\documentclass[journal]{IEEEtran}

\usepackage{amsmath,amssymb,amsfonts}
\usepackage{soul}
\usepackage{booktabs}
\usepackage{graphicx}
\usepackage[table]{xcolor}

\usepackage{pgfplots}
\usepackage{pgfplotstable}
\pgfplotsset{compat=1.18}

\usepackage{amsthm}
\usepackage{placeins}

\DeclareMathOperator{\R}{\mathbb{R}} 
\newcommand{\N}{\mathbb{N}} 

\newcommand{\cA}{{A}} 

\newcommand{\Fib}{\text{Fib}}

\colorlet{myblue}{cyan!40}
\definecolor{codegreen}{rgb}{0,0.6,0}

\usepackage{wrapfig}

\usepackage[floatrowsep=quad]{floatrow}

\usepackage{float}
\usepackage{stfloats} 
\usepackage{algorithm}
\usepackage{algpseudocode}

\usepackage{dsfont}
\usepackage{multirow}
\usepackage{pifont}

\usepackage[caption=false,font=footnotesize]{subfig}

\usepackage{mdframed}
\newcommand{\myNum}[1]{(\emph{#1})}
\newcommand{\smartparagraph}[1]{\vspace{2pt} \noindent {\bf #1}}

\newtheorem{theorem}{Theorem}
\newtheorem{lemma}[theorem]{Lemma}

\newcommand{\wmax}{w_{\text{max}}}
\newcommand{\wmin}{w_{\text{min}}}

\usepackage{xspace}
\newcommand{\modelname}{Fibottention\xspace}

\newcommand{\eg}{\emph{e.g.}\xspace}
\newcommand{\ie}{\emph{i.e.}\xspace}

\DeclareRobustCommand{\!}{%
  \ifmmode\mskip-\thinmuskip\else\hspace{-0.1667em}\fi
}

\usepackage[breaklinks,colorlinks]{hyperref}
\hypersetup{
  allcolors = {blue!60!black},
  pdftitle={Fibottention: Inceptive Visual Representation Learning with Diverse Attention Across Heads}
}
\usepackage{orcidlink}
\usepackage{arydshln}

\usepackage{cleveref}

\usepackage{tabularx} 
\usepackage{array}    
\newcolumntype{C}{>{\centering\arraybackslash}X} 

\usepackage{makecell}

\usepackage{silence}
\makeatletter
\DeclareFontShape{OT1}{ptm}{m}{scit}{<-> ssub * ptm/m/it}{}
\makeatother

\begin{document}
\bstctlcite{IEEEexample:BSTcontrol}

\title{Fibottention: Inceptive Visual Representation Learning with Diverse Attention Across Heads}

\author{%
\textbf{Ali K. Rahimian}$^{1}$ \quad \textbf{Manish K. Govind}$^{1}$ \quad \textbf{Subhajit Maity}$^2$ \quad \textbf{Dominick Reilly}$^1$ \\
\textbf{Christian Kümmerle}$^{2\dagger}$ \quad
\textbf{Srijan Das}$^{1\dagger}$ \quad \textbf{Aritra Dutta}$^{2\dagger}$ \\
$^1$University of North Carolina at Charlotte \quad $^2$University of Central Florida \\
\texttt{\{akhalegh, sdas24\}@charlotte.edu} \quad \texttt{ \{kuemmerle, aritra.dutta\}@ucf.edu}


\thanks{$\dagger$ Equal contribution as Project Lead. \par}
}
    



\maketitle

\begin{abstract}
Vision Transformers and their variants have achieved remarkable success in diverse visual perception tasks. Despite their effectiveness, they suffer from two significant limitations. First, the quadratic computational complexity of multi-head self-attention (MHSA), which restricts scalability to large token counts, and second, a high dependency on large-scale training data to attain competitive performance.
In this paper, to address these challenges, we propose a novel sparse self-attention mechanism named \emph{\modelname}. \modelname employs structured sparsity patterns derived from the Wythoff array, enabling an $\mathcal{O}(N \log N)$ computational complexity in self-attention. By design, its sparsity patterns vary across attention heads, which provably reduces redundant pairwise interactions while ensuring sufficient and diverse coverage. This leads to an \emph{inception-like functional diversity} in the attention heads, and promotes more informative and disentangled representations. We integrate \modelname into standard Transformer architectures and conduct extensive experiments across multiple domains, including image classification, video understanding, and robot learning. Results demonstrate that models equipped with \modelname either significantly outperform or achieve on-par performance with their dense MHSA counterparts, while leveraging only $2\%$ of all pairwise interactions across self-attention heads in typical settings, 
resulting in substantial computational savings. Moreover, when compared to existing sparse attention mechanisms, \modelname consistently achieves superior results on a FLOP-equivalency basis. Finally, we provide an in-depth analysis of the enhanced feature diversity resulting from our attention design and discuss its implications for efficient representation learning. Reproducible code is available at \url{https://github.com/Charlotte-CharMLab/Fibottention}.
\end{abstract} 

\begin{IEEEkeywords}
Sparse attention, representation learning, Transformers, visual encoding, head diversity.
\end{IEEEkeywords}


\section{Introduction}
\label{sec:introduction}
\providecommand{\eg}{\emph{e.g.}\xspace}
\providecommand{\ie}{\emph{i.e.}\xspace}
\providecommand{\modelname}{\textsc{Fibottention}\xspace}

Transformer-based architectures, such as large foundation models, \eg\!\!, GPT~\cite{gpt,gpt3}, BERT~\cite{BERTnlp}, ALBERT~\cite{Lan2020ALBERT}, ViT~\cite{dosovitskiy2020vit}, DETR~\cite{DETR}, D-DETR~\cite{D-DETR}, CLIP~\cite{clip_representation}, have achieved dominating performance in many downstream tasks such as object detection and tracking \cite{dutta2024multiview}, document summarization \cite{gpt}, language modeling \cite{BERTnlp}, and video understanding \cite{reilly2023just}. Compared to other deep neural networks, such as convolutional neural networks, Transformers excel when trained on large-scale datasets with extensive model parameters \cite{dosovitskiy2020vit}. However, their performance typically degrades in low-data regimes~\cite{dosovitskiy2020vit}. 
The growing demand for models that can be deployed and trained effectively on diverse edge devices or within the Internet of Things (IoT) \cite{agarwalla2024enabling, reidy2023efficient, sun2024towards, tuli2023edgetran, qu2022dota} has driven significant research interest in developing efficient models 
that are both compute-efficient and data-efficient. 

At the core of Transformer-based models is the multi-head self-attention (MHSA) \cite{Kim2016structured, attention,dosovitskiy2020vit} mechanism. In MHSA, $N$ input feature vectors (a.k.a. tokens) of dimension $d$ are mapped to $h$ query, key and value matrices $Q_i, K_i, V_i \in \R^{N \times d_h}$, for each head $i=1,\ldots, h$, which are subsequently mapped to output feature vectors. Its computational bottleneck is the computation of the $h$ attention matrices $\cA_i = Q_i K_i^{\top} / \sqrt{d_h} \in  \R^{N \times N}$ consisting of inner products of queries and keys, which inherently limits the number of tokens due to the resulting $O(N^2)$ complexity. 
To mitigate this limitation, a considerable body of literature considers Transformer variants that evaluate the attention matrices only at a \emph{sparse subset }
of their entries $\Omega\subset [N] \times [N]$ of size $s = |\Omega| < N^2$, thereby lowering the incurred time complexity from $O(N^2)$ to $O(s)$. It has been observed that, while reducing the computational load, such \emph{sparse attention} patterns can degrade the modeling capacity of the underlying architecture and often reduce their accuracy relative to dense attention on several benchmarks, highlighting a trade-off between efficiency and performance \cite{SparseTransformer19,yun2020n, Beltagy2020LongformerTL, zaheer2020big, starformer}. Since MHSA is designed to facilitate interactions between all tokens, finding suitable choices for such a global support set $\Omega$ with a favorable trade-off between efficacy and efficiency is challenging due to the data, model, and instance dependence of the attention values.

Popular sparse attention strategies include local attention with sliding windows of fixed window size \cite{Wang-2019MultiBERT,Beltagy2020LongformerTL,ramachandran2019stand}, in which only interactions between spatially proximal tokens are considered, often augmented by random attention mechanisms \cite{zaheer2020big,zhang2023vision}.
Interestingly, the majority of sparse MHSA mechanisms are designed in the context of natural language processing, with the notable exception in visual domains for high-resolution images \cite{Esser-2021taming,zhang2021multi}. Unlike natural language, images and videos exhibit strong spatial and spatiotemporal redundancy: neighboring pixels often encode similar content, and many key–query interactions in dense MHSA are unnecessary for visual representation learning. This redundancy means uniform, all-to-all attention spends substantial compute on low-utility interactions. 

In either domain, sparse attention has been mostly considered as a modification of MHSA that \emph{harms} model performance while improving computational feasibility. In this work, however, we distill insights from a substantial body of literature on sparse attention mechanisms to design a versatile, deterministic sparse attention pattern that is able to improve the model performance of Transformers in a regime of limited visual data, while simultaneously improving their computational efficiency across dataset sizes. 

In particular, we propose \emph{\modelname}, an MHSA variant that can be used as a drop-in replacement for full attention in vision Transformers
\cite{dosovitskiy2020vit, cvt, timesformer, mvitv2}.
For each attention head $A_i$, \modelname{} deploys a complementary sparsity
pattern $\Omega_i$ based on \emph{non-overlapping, dilated sliding windows}
\cite{Beltagy2020LongformerTL} with a \emph{growing dilation} schedule to
capture both local and global token interactions within each head.
Concretely, we implement this using \emph{non-overlapping generalized Fibonacci
sequences} as dilation sequences $(f_n)_n$ across heads (see Table \ref{tab:mod-wythoff-array}, Fig. \ref{fig:arch}, and Eq. $\eqref{def:Omegai:Fibottention}$).
Fibonacci-like schedules progressively thin out connections as distance grows,
encoding dense local interactions and increasingly sparse long-range links. This structured sparsity is motivated not only by compute efficiency, but also
by an inductive bias that promotes multi-scale feature aggregation and head-wise
complementarity.
The design choices of \modelname~are based on the following key insights: \myNum{i} \textit{the principal diagonal} of an attention matrix $A_i$ might not contain helpful information for the model~\cite{shi2021sparsebert}; that \myNum{ii} a structured, deterministic sparsity pattern $\Omega$, which is able to capture both local and global token interactions is desirable from a modeling \cite{zaheer2020big, starformer} and efficiency perspective; \myNum{iii} sparsity patterns, $\Omega_1,\ldots,\Omega_h$ that \textit{differ across} the attention heads, $A_i$ can achieve a diversity of feature representations across heads (analogous to the intuition behind CNN-based \emph{Inception} models \cite{szegedy2016inception}); and finally, that \myNum{iv} a \emph{low overlap} between the $\Omega_i$ is desirable as it has the potential to \emph{maximize} the diversity of the resulting feature representations while \emph{minimizing} the total number $\sum_{i=1}^h |\Omega_i|$ of inner products to be calculated. Consequently, the \modelname's structured attention increases the feature diversity across heads and exhibits a performance-enhancing inductive bias that is particularly beneficial for visual domains with {\em limited training data} (e.g., video understanding and robotics). 


We extensively evaluate \modelname~in conjunction with diverse state-of-the-art Transformer architectures catered towards visual representation tasks, including image classification, video action recognition, and robot imitation learning (\S\ref{sec:numerical}).
On CIFAR-10/100~\cite{cifar}, it consistently surpasses Transformer baselines trained with full multi-head self-attention (MHSA) while remaining on par when trained on ImageNet-1K~\cite{imagenet} (Table~\ref{tab:sota}).
As we show in Section \ref{sec:experimental:results}, \modelname's behavior is consistent across various Transformer architectures (Table~\ref{tab:backbones}): when incorporated into UPop~\cite{shi2023upop}, iFormer~\cite{zheng2025iformer_iclr}, Swin-B~\cite{liu2021swin}, and ConViT-B~\cite{d2021convit}, it consistently reduces attention compute to a small fraction of the baseline while preserving or improving accuracy depending on capacity and training regime, a feat that other sparse attention designs such as top-$K$ attention \cite{gupta2021memoryefficient} or Longformer \cite{Beltagy2020LongformerTL} fail to achieve. 
In temporal domains, we see that \modelname improves the Top-1 accuracy of TimeSformer~\cite{timesformer} on SmartHome~\cite{smarthome} and NUCLA~\cite{nucla} (Table~\ref{tab:video}). For robot imitation learning tasks (Lift/Can/PushT), it achieves the best average task completion rates (Table~\ref{tab:robotics}). Overall, we observe that \modelname{} can achieve large, quantifiable reductions in attention FLOPs while maintaining or enhancing predictive performance across scales, modalities, and backbones.


The remainder of this paper is organized as follows. Section \ref{sec:related_work} reviews related work on sparse, adaptive, and diverse attention for vision Transformers. Section \ref{sec:method} details the design of \modelname based on Fibonacci dilation sequences defined from the Wythoff array and its head-wise masking strategy. Section \ref{sec:numerical} presents empirical evaluations of vision Transformers modified by \modelname and other sparse attention variants on image classification tasks. Section \ref{sec:fibottention:othervisual} extends \modelname to other visual domains, including video action recognition and robot imitation learning, describing training protocols and empirical findings, and Section \ref{sec:further-ablation} provides ablations, studies of the impact of \modelname on head diversity and inductive bias. Section \ref{sec:conclusion} concludes with limitations and avenues for future work. Finally, we provide a proof of \modelname's $O(N \log N)$ time complexity in Appendix A, 
implementation and hyperparameter details in Appendices B 
and C,  
and further ablation studies in Appendices D 
and E.
\section{Related Work}
\label{sec:related_work}

\smartparagraph{Vision Transformers.}
Derived from Transformers~\cite{attention}, which excel on long-range sequence tasks in NLP, the vision Transformer (ViT)~\cite{dosovitskiy2020vit} splits images into small patches, each corresponding to a token, and has emerged as a popular architecture in visual understanding tasks.
While ViTs~\cite{dosovitskiy2020vit,deit} outperform CNN-based models in a variety of visual representation tasks, they require extensive training data to achieve this superior performance and exhibit a quadratic time complexity in the number of tokens $N$. DeiT models~\cite{deit,wu2022tinyvit} exhibit advantages over ViT in the presence of limited training data due to their knowledge distillation, but still require a quadratic time complexity with respect to $N$.
Another line of research, which includes CvT~\cite{cvt} and ConViT~\cite{d2021convit} models, has achieved strong performance based on hybridization with convolutions. Mobile-oriented hybrids such as iFormer~\cite{zheng2025iformer_iclr} push this direction toward latency-constrained regimes by coupling ConvNeXt-style local processing with lightweight modulation attention, achieving favorable accuracy--latency tradeoffs on smartphone-class hardware. A useful inductive bias is also provided by hierarchical models such as MViTv2 \cite{mvitv2} or Swin Transformer~\cite{liu2021swin}, 
which combine a patch merging strategy with MHSA mechanisms adapted to blocks of tokens. While all these models are compatible with the MHSA modification provided by \modelname, we focus in our experiments on models without knowledge distillation for a more controlled experimental setup.

\smartparagraph{Sparse Attention.}
%
Despite their advancements, ViT variants share a common limitation: the MHSA mechanism inherently requires the evaluation of $O(N^2)$ token interactions, posing a significant computational challenge. However, theoretical insights \cite{yun2020n} demonstrate that sparse attention, given an appropriate sparsity pattern, can effectively approximate any sequence-to-sequence function, mirroring the capabilities of full attention \cite{yun2019transformers}. Several studies, primarily in NLP, have explored optimal sparsity patterns, emphasizing the importance of central diagonal elements in $\Omega_i$ \cite{clark2019does, kovaleva2019revealing}. Longformer \cite{Beltagy2020LongformerTL}, BigBird \cite{zaheer2020big}, and Star-Transformer \cite{guo2019star} incorporate global and local token interactions to enhance performance \cite{kovaleva2019revealing, li2019enhancing}. Regional token interactions are commonly implemented as a diagonal sliding window within $\Omega$, which can be expanded through dilated sliding window attention to increase the receptive field without additional computational overhead \cite{Beltagy2020LongformerTL, hassani2022dilated}. Longformer further extends this by incorporating random token pair interactions, while sparse Transformers \cite{SparseTransformer19} introduce specialized sparse patterns designed for efficient long-sequence generation. Collectively, these methods leverage combinations of local, global, sliding window, dilated sliding window, and random attention patterns. Later, Shi et al.\ \cite{shi2021sparsebert} observed that the main diagonal elements in $\Omega_i$ are redundant, proposing a learnable differentiable attention mask to refine the sparsity structure.
Overall, compared to NLP, structured sparse attention is less explored in vision problems, in which token interactions tend to be more redundant and contextually distinct from the language domain, and have been mostly investigated in the context of high-resolution setups \cite{Esser-2021taming,zhang2021multi,li2025radial}, e.g., via the hierarchical neighborhood attention Transformer \cite{hassani2023neighborhood}; while in these works, computational efficiency is the main focus of the sparse attention design, we jointly optimize for inductive bias in \modelname. 
In Section \ref{sec:experimental:results}, we use a selection of the above-mentioned sparse attention modifications (adapted to the vision setting) as baselines in empirical comparisons to the adaptation of \modelname into Vision Transformers. Beyond sparse attention, other full attention approximations have been studied, with similar potential for breaking its quadratic complexity bottleneck, such as efficient attention \cite{shen2021efficient} and Linformer \cite{wang2020linformer}, which apply non-linear transformations of key and query matrices instead of softmax and low-rank approximations of keys and values, respectively.

\smartparagraph{Adaptive and Diverse Attention.}
While many sparse attention designs fix attention patterns across batches, several works explore adaptive or instance-dependent sparsity. For example, \cite{kitaev2020reformer,roy2021efficient,Wei2023-Sparsifiner} have proposed learned instance-dependent attention masks, which can be effective but impose additional model complexity dedicated to the learning of the sparsity mask. A related line of work studies variants of \emph{top-$k$ attention} \cite{Zhao2019explicit,gupta2021memoryefficient,Sander2023fast,you2025spark}, where each query attends dynamically only to the $k$ most relevant keys, resulting in instance-dependent attention computations. Recently, top-$k$ sparse attention in state-of-the-art large-scale language models \cite{liu2025deepseek} coupled with a lightweight indexer module has enabled efficient Transformer architectures applicable for long contexts. \modelname, unlike learned sparse or top-$k$ sparse attention mechanisms, is non-adaptive and does not require a conceptual and computational overhead due to its fixed sparsity patterns.
A limited number of works report observations of improved empirical performance of Transformers using attention patterns that \emph{vary across heads}; examples of such works are Longformer~\cite{Beltagy2020LongformerTL}, which reports improved performance when combining sparse heads with and without dilation in their architectures, and \cite{SparseTransformer19}, which provides evidence that differently sized sub-blocks across heads are preferable. However, we are not aware of works on fixed sparse attention patterns that systematically incorporate these observations into their mechanism design, especially in vision settings \cite{zhang2021multi}, as \modelname does with its complementary diverse sparse design. With adaptive designs, however, the benefits of diverse attention have recently been studied in language models for inference \cite{Fu2025mixture} or through the introduction of gating \cite{Qiu2025gated}.

\smartparagraph{Dynamic Token Sparsification.}
Several methods have explored dynamic token reduction strategies to further mitigate the quadratic complexity of ViTs. DynamicViT~\cite{rao2021dynamicvit} introduces lightweight prediction modules that estimate the importance of each token at intermediate layers, progressively discarding less informative ones while maintaining the most relevant content for recognition. 
PS-ViT~\cite{yue2021psvit} adopts a progressive sampling strategy: at each iteration, the model predicts new sampling offsets to refine where tokens should be drawn, enabling it to concentrate on discriminative image regions. EViT~\cite{liang2022evit} reorganizes tokens by identifying attentive and inattentive ones through class-token attention, retaining the former while merging the latter into compact representations for subsequent layers. Beyond vision-only settings, UPop~\cite{shi2023upop} extends progressive pruning to vision-language Transformers. While these methods dynamically adapt the token set during inference, \modelname sparsifies only token \emph{interactions}, not tokens. Thus, \modelname's design is orthogonal to dynamic token reduction approaches and can be combined with token sparsification; in combination with those ideas, \modelname has the potential to further improve runtime efficiency and inductive biases of token-sparsified Transformer models.

\section{Method}
\label{sec:method}

\begin{figure*}[t]
  \centering
  \includegraphics[width=0.8\textwidth]{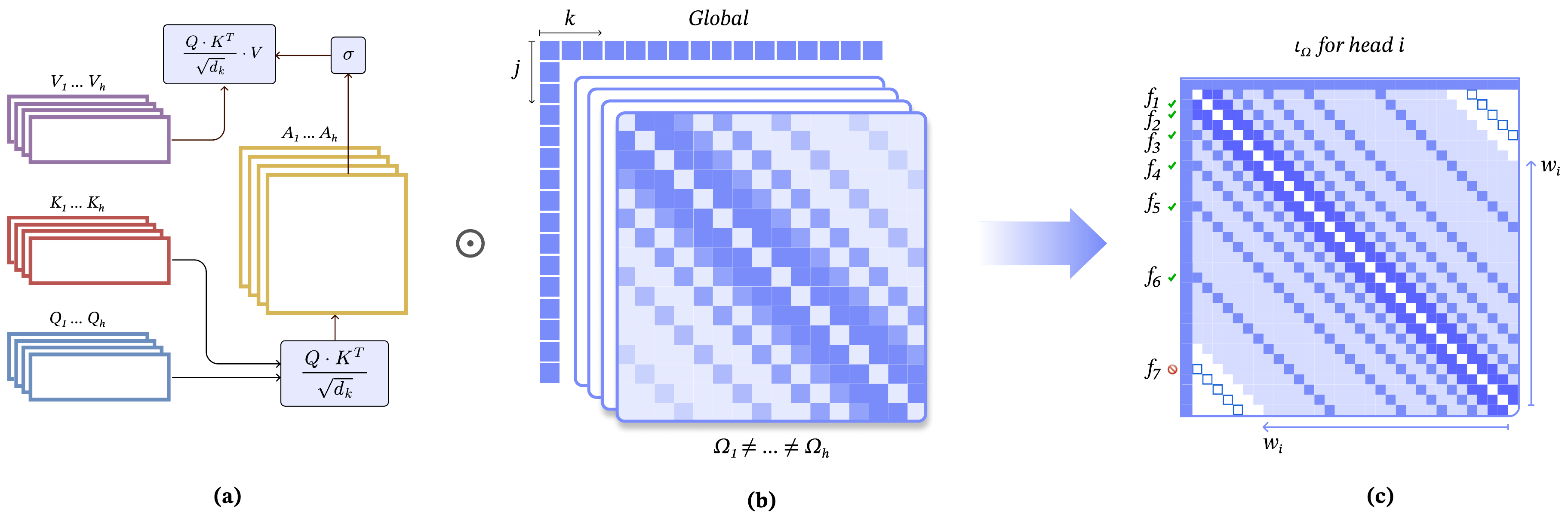}
  \caption{{(a) The MHSA. (b) A general sparse attention computation strategy. A sequence of sparse support sets, $\{\Omega_i\}_{i=1}^h$, where each set selects $|\Omega_i|<N^2$ entries of the attention matrix. (c) The generalized masking strategy of \modelname~that controls sparsity of each attention matrix $\cA_{i}$ through a dilated sequence, $(f_n)_n\subset \mathbb{N}$, and a fixed window size, $w$ for each head. Elements on \(f_1\) and \(f_2\) occur exclusively in the Modified Wythoff variant.}}
  \label{fig:arch}
\end{figure*}

In this section, we discuss the multi-head self-attention modification dubbed \modelname, which uses structured, diverse support sets derived from Fibonacci–Wythoff sequences in its sparse attention evaluation in each head. We start with a general sparse attention framework that allows formulating and comparing different sparsity patterns for attention heads $\{\cA_i\}_{i=1}^h$.

\subsection{Sparse Attention with Windowed Dilation}
\label{sec:general}

Instead of observing the attention matrix $A_i$ at each of its $N^2$ entries, sparse attention mechanisms compute only the dot products whose indices are supported on a subset $\Omega \subset \{1,2,\ldots, N\}^2$, i.e., we can define the sparse attention matrix $\cA_i^\Omega \in \mathbb{R}^{N \times N}$ of the $i$-th head corresponding to mask $\Omega$ as
\begin{equation} \label{eq:Omegamasking}
(\cA_i^{\Omega})_{j,k} = \begin{cases} \frac{Q_{i}^{(j)\top} K_{i}^{(k)}}{\sqrt{d_h}}, & \text{ if } (j,k) \in \Omega, \\
-\infty, & \text{ if }(j,k) \notin \Omega;
\end{cases}
\end{equation}
for any $j,k \in [N]$, where $Q_{i}^{(j)} \in \R^{d_h}$ and $ K_{i}^{(k)} \in \R^{d_h}$ are the $j$-th query vector and the $k$-th key vector of the $i$-th attention head, respectively. If $\odot$ denotes the entrywise matrix multiplication, also called Hadamard product, this can be written as $\cA_i^{\Omega} = \operatorname{sign}(\cA_i) \odot (|\cA_i| \odot \iota_{\Omega})$, where $\iota_{\Omega} \in \R^{N \times N}$ is an indicator matrix of the index set $\Omega$ that is $1$ for indices $(j,k) \in \Omega$ and $-\infty$ otherwise.  
{In this work, we study structured support sets that capture both local and global interactions while ensuring efficient inference and training through sparsity. To this end, we introduce the notion of a} \emph{dilation sequence}, $(f_n)_n \subset \N$, which determines the sequence of distances between indices of tokens that attend to each other and render diversity across heads. Furthermore, for a given attention head, we fix a \emph{window size}, $1 \leq w \leq N$, which, independently of the dilation sequence, provides an upper bound for the index distance between interacting token indices in the attention matrix.

Given the sequence, $(f_n)_n$ and parameter, $w$, we define the \emph{support set}, $\Omega_{w}^{(f_n)}\subset \{1,2,\ldots,N\}^2$ of \emph{interacting query-key pairs dilated by $(f_n)_n$ of window size $w$} such that
\[
\Omega_{w}^{(f_n)} = \big\{(j,k): |j-k| \in \{f_n\}_{n}, |j-k| \leq w \big\}.
\]
We refer to Figure \ref{fig:arch}(c) for a visualization of such support sets; $\Omega = \Omega_{w}^{\{f_n\}}$ represents the effective set of indices of query-key pairs for which we need to calculate the dot product in a given attention head $\cA_i$.

Several dilation sequences have been studied in both vision and language Transformer architectures. The majority of existing literature \cite{SparseTransformer19,Beltagy2020LongformerTL,zhang2021multi,hassani2022dilated} considers dilation sequences that are multiples of a fixed factor $c \in \N$, i.e., $(f_n)_n = (c n)_{n \in \N}$, corresponding to sliding windows with constant dilation factor $c$. While providing a certain level of efficiency, their attention complexity only reduces from $O(N w)$ to $O(N w/c)$, which is still of order $N^2$ if the window size $W$ is chosen to be $w=O(N)$. On the other hand, choosing a small window size $w = O(1)$ prevents the inclusion of any global interactions. Dilation patterns based on different dilation sequences have been less explored; \cite{li2019enhancing} studied exponentially dilated sequences giving rise to attention complexities of order $O(N \log w) = O(N \log N)$.

\begin{table}[ht]
\captionsetup[table]{position=top}
\setlength{\tabcolsep}{6pt}
\caption{Generalized Fibonacci sequences $\Fib(a_i^{\text{Wyt}},b_i^{\text{Wyt}})$ and $\Fib(a_i^{\text{Wyt-m}},b_i^{\text{Wyt-m}})$ drawn from the Wythoff array used by head $i$ in \modelname{} (default and modified variants). \label{tab:mod-wythoff-array}}
\centering
\begin{tabular}{
    c
    |
    c
    c
    :
    c
    c
    c
    c
    c
    c
}
\toprule
$i$ &
$\boldsymbol{a_i^{\text{Wyt-m}}}$ &
$\boldsymbol{b_i^{\text{Wyt-m}}}$ &
$\boldsymbol{a_i^{\text{Wyt}}}$ &
$\boldsymbol{b_i^{\text{Wyt}}}$ &
\multicolumn{4}{c}{} \\
\midrule
1 & 0 & 1 & 1 & 2 & 3 & 5 & 8  & $\dots$ \\
2 & 1 & 3 & 4 & 7 & 11 & 18 & 29 & $\dots$ \\
3 & 2 & 4 & 6 & 10 & 16 & 26 & 42 & $\dots$ \\
4 & 3 & 6 & 9 & 15 & 24 & 39 & 63 & $\dots$ \\
5 & 4 & 8 & 12 & 20 & 32 & 52 & 84 & $\dots$ \\
6 & \vdots & \vdots & \ldots & \ldots & \ldots & \ldots & \ldots & $\ldots$ \\
\bottomrule
\end{tabular}%
\end{table}

\subsection{Diverse Sparse Attention through Wythoff–Fibonacci Dilation Sequences}
\label{sec:fibottention}

While support sets derived from exponential dilation sequences lead to sparse attention matrices, it might happen that crucial query-key interactions are not captured by overly sparse patterns, deteriorating the quality of the resulting MHSA representations. At the same time, limited experimental results in \cite{Beltagy2020LongformerTL,kovaleva2019revealing,Ding-2023Longnet} indicate that varying support set patterns across attention heads can improve model performance. Furthermore, state-of-the-art sparse attention mechanisms aim for a delicate balance between covering local and global interactions \cite{Beltagy2020LongformerTL,zaheer2020big}, and do not necessarily include the interactions on the main diagonal of the attention matrix \cite{shi2021sparsebert}.

Motivated by these observations, we postulate that \emph{sparse attention matrices with diverse, well-designed support patterns across attention heads} are desirable and have the potential to lead to improved learned representations.

\smartparagraph{Fibonacci Dilation Sequences.} To achieve our design goals, we propose a sparse attention pattern that builds on \emph{(generalized) Fibonacci sequences} \cite{Sigler2003fibonacci,Koshy2019fibonacci}. The well-known Fibonacci sequence $(f_n)_{n \in \N}$ is defined as the sequence of integers $(0,1,1,2,3,5,8,13,\ldots)$ \cite{oeis-fibonacci} satisfying the linear recurrence relation
\setlength{\abovedisplayskip}{3pt}
\setlength{\belowdisplayskip}{3pt}
\begin{equation} \label{eq:recurrence}
f_{n+1} = f_n + f_{n-1},
\end{equation}
for each $n \geq 2$, where $f_1= 0$ and $f_2 =1$. Binet's formula \cite{Koshy2019fibonacci} states that the $n$-th Fibonacci number satisfies $f_n = (\phi^{n-1}-\psi^{n-1})/\sqrt{5}$, where $\phi= (1+\sqrt{5})/2 \approx 1.618$ is the golden ratio and $\psi= (1-\sqrt{5})/2$. From this formula, it can be inferred that after initial slow growth, the sequence grows exponentially with respect to the base $\phi$. Similar integer sequences can be defined from the recurrence \eqref{eq:recurrence} by fixing the initial elements, $f_1 = a \in \N$ and $f_2 = b \in \N$. 

Given a window size $w$, parameters, $a,b \in \N$, and denoting the corresponding \emph{generalized Fibonacci sequence}, $(f_n)_{n}$ by $\Fib(a,b)$, we can define a corresponding support set for an $N \times N$ attention matrix as $\Omega_{w}^{\Fib(a,b)} = \big\{(j,k): |j-k| \in \Fib(a,b), |j-k| \leq w \big\}$. 
An experimental ablation study (see Section \ref{subsec:impact-dilation}) indicates that a simple Fibonacci attention pattern can already be advantageous compared to other dilation sequences, even when deployed uniformly across heads.

\smartparagraph{Wythoff Array and its Properties.}
Among integer sequences based on order-$2$ linear recurrence relations, generalized Fibonacci sequences are attractive for creating attention support sets since by varying $a$ and $b$, a variety of integer values can be covered while retaining the same long-term growth rate (see Appendix A, Lemma 1) 
as the Fibonacci numbers. Accordingly, we use $h$ different Fibonacci-type sequences, $\Fib(a_i,b_i)$ with different initial values $a_1,\ldots,a_h \in \N$ and $b_1,\ldots,b_h \in \N$, giving rise to \emph{head-specific} attention support sets. Defining also head-specific window sizes, $w_1,\ldots,w_h \leq N$, we obtain the support set $\Omega_i$ for the $i$-th attention head matrix $A_i$ defined as $\Omega_{w_i}^{\Fib(a_i,b_i)}$ for each head index $i=1,\ldots,h$.

Within this framework, we aim to choose the sequence parameters $(a_i,b_i)_i$ such that the following three \emph{desiderata} are satisfied: \myNum{i} the overlap between different attention head support sets should be minimized, allowing for a semantic specialization of the corresponding head weights during training, \myNum{ii} the total size $\sum_{i=1}^h |\Omega_{w_i}^{\Fib(a_i,b_i)}|$ of the support sets should be small to retain efficiency, but within that constraint, \myNum{iii} as many \emph{relevant} query-key interactions as possible should be captured by at least one attention head, i.e., the set union $\cup_{i=1}^h \Omega_{w_i}^{\Fib(a_i,b_i)}$ should be maximal. 

A suitable, essentially hyperparameter-free choice can be derived from the Wythoff array \cite{Morrison-1980Stolarsky,Conway-2016Extra,Chen-2022GeneralizingWythoff}, which had been originally introduced in the context of a combinatorial game \cite{Wythoff-1907Modification}. The Wythoff array can be considered as a collection of generalized Fibonacci sequences $\{\Fib(a_i,b_i)\}_{i \in \N}$ with specific choices $a_i^{\text{Wyt}}$ and $b_i^{\text{Wyt}}$ for each $i \in \N$ that have provably \emph{no overlap}, but contain each integer exactly once \cite{Morrison-1980Stolarsky,Conway-2016Extra}. In particular, the $i$-th row sequence of the Wythoff array is given by the sequence, $\Fib(a_i^{\text{Wyt}},b_i^{\text{Wyt}})$ with initial elements, $a_i^{\text{Wyt}} = \lfloor \lfloor i \phi\rfloor \phi \rfloor$ and $b_i^{\text{Wyt}} = \lfloor\lfloor i \phi\rfloor \phi^2 \rfloor$; see Table \ref{tab:mod-wythoff-array}. 

\smartparagraph{\modelname.}
Based on the above considerations, we define a novel non-adaptive sparse attention mechanism, called \emph{\modelname}, that is designed as a drop-in replacement of full self-attention in multi-head self-attention blocks. In any given MHSA layer with $h$ heads, for a given head index $i = 1 \ldots, h$, we restrict the computation of unnormalized attention weights in $A_i \in \R^{N \times N}$ to the support set  $\Omega_i:= \Omega_{w_i}^{\Fib(a_i^{\text{Wyt}},b_i^{\text{Wyt}})}$ given as
\begin{equation} \label{def:Omegai:Fibottention}
\big\{(j,k): |j-k| \in \Fib(a_i^{\text{Wyt}},b_i^{\text{Wyt}}), |j-k| \leq w_i \big\},
\end{equation}
where the window size $w_i$ of the $i$-th head is based on two model-wide hyperparameters $\wmin$ and $\wmax$, which are chosen based on insights into the modality of the task and the data distribution. Specifically, we choose $w_i$ based on the formula, 
$w_i = \wmin + \Big\lfloor \frac{\wmax-\wmin}{h-1} (i-1) \Big\rfloor$ for $i =1,\ldots,h$, which linearly interpolates between $\wmin \leq N$, the \emph{minimal window} \emph{size bound across heads}, and the \emph{maximal window size bound across heads} $\wmax$ satisfying $\wmin \leq \wmax \leq N$. The resulting \emph{spacing} of window sizes across heads is designed to further diversify the representations learned across heads as, in the case of a large disparity between $\wmin$ and $\wmax$, heads with lower indices $i$ are biased to encode more local information, whereas heads with $w_i \approx \wmax$ are biased towards incorporating more global interactions. Following \eqref{eq:Omegamasking}, we define \modelname's sparse attention matrices as $\cA_i^{\Omega_i} = \cA_i \odot \iota_{\Omega_i}$ for each $i=1,\ldots,h$ with $\Omega_i$ satisfying \eqref{def:Omegai:Fibottention}.

For Transformer architectures with several MHSA layers, we further require that the head-wise support sets are shuffled along the layer so that the $i$-th head uses the sets,  $\{\Omega_{\pi(1)},\ldots,\Omega_{\pi(h)}\}$ within \modelname, where $\pi: [h] \to [h]$ is a random permutation function (fixed for each layer). We refer to Appendix B 
for a formal outline.

\smartparagraph{Modified Wythoff Array.} While we observe excellent performance of vanilla \modelname in image classification tasks (see Section \ref{sec:numerical}), its performance degrades in tasks in other domains due to its high degree of sparsity, which might not always capture well enough important local interactions. For such cases, we propose a variant of this sparse attention mechanism that \emph{includes two predecessor sequence elements} into each Wythoff row sequence $\Fib(a_i^{\text{Wyt}},b_i^{\text{Wyt}})$; following the recurrence \eqref{eq:recurrence}, we can define new initial sequence elements $b_i^{\text{Wyt-m}} = b_i^{\text{Wyt}} - a_i^{\text{Wyt}}$ and $a_i^{\text{Wyt-m}} =  a_i^{\text{Wyt}} -b_i^{\text{Wyt-m}}$ and support sets $\Omega_i= \Omega_{w_i}^{\Fib(a_i^{\text{Wyt-m}},b_i^{\text{Wyt-m}})}$ for each head index $i$. Unlike for the original Wythoff array, it is not the case anymore that the resulting sequences contain each integer only at most once \cite{Morrison-1980Stolarsky,Conway-2016Extra}; on the other hand, it can be proven that this modified {\modelname} shares each query-key interaction pair only across at most \emph{three} heads \cite{Conway-2016Extra}. The differences in the resulting support set patterns are visualized in Figure \ref{fig:dilated-fixed-gap} and Table  \ref{tab:mod-wythoff-array}. We refer to Appendix B 
for a detailed outline that includes both the (default) Wythoff and the Modified Wythoff variants of \modelname. 

In Appendix A-C,
we provide a proof that the total computational effort for inference in both \modelname~variants requires the computation of only $O(N \log(\wmax))$ token interactions.



\section{\modelname~for Image Classification}\label{sec:numerical}
In this section, we evaluate the performance of \modelname~across various image classification tasks in conjunction with different Transformer architectures, and compare the design to state-of-the-art sparse and efficient attention designs.

\subsection{Experimental Setup}
\noindent \textbf{Datasets.}
We report Top-1 accuracy on CIFAR-10 (C10)~\cite{cifar}, CIFAR-100 (C100)~\cite{cifar}, Tiny-ImageNet~\cite{imagenet}, and ImageNet-1K (IN-1K)~\cite{imagenet}.

\noindent \textbf{Training.} For training \modelname~ and other sparse attention methods with various Transformer architectures, we use the training recipe of DeiT~\cite{deit}. All models are trained from \textit{random initialization} for $100$ epochs with an effective batch size of $64$ using four 48GB A6000 GPUs, employing ViT-Base (ViT-B)~\cite{dosovitskiy2020vit,deit} unless otherwise specified. The hyperparameters are set as $w_{\rm min} = 5$, and $w_{\rm max} = 65$ unless otherwise stated.
\vspace*{-5pt}

\subsection{Experimental Results} \label{sec:experimental:results}

\begin{table}[t]
\centering
\small
\setlength{\tabcolsep}{3pt}
\renewcommand{\arraystretch}{1.05}
\resizebox{\linewidth}{!}{%
\begin{tabular}{lcccccc}
\toprule
\multirow{2}{*}{\textbf{Method}} &
\multicolumn{4}{c}{\textbf{Top-1 Accuracy (\%) $\uparrow$}} &
\multirow{2}{*}{\makecell[c]{\textbf{Pruning} \\ \textbf{Ratio $\uparrow$}}} \\
\cmidrule(lr){2-5}
& C10 & C100 & Tiny-IN & IN-1K & \\
\midrule
Full Attention~\cite{attention}                             & 84.2   & 59.4   & \underline{75.2}   & \textbf{75.9} & 0\% \\
\midrule
Random Attention~\cite{zaheer2020big}       & 80.7   & 56.5   & 69.4   & 68.7 & 98.52\% \\
Top-$k$ Attention~\cite{gupta2021memoryefficient} & 81.1   & 57.1   & 72.9   & 73.4 & 98.48\% \\
Sparse Transformer~\cite{SparseTransformer19} & 81.3   & 58.2   & 70.3   & 68.7 & 98.47\% \\
BigBird~\cite{zaheer2020big}                & 86.8   & 63.4   & 73.4   & 71.5 & 97.96\% \\
Longformer~\cite{Beltagy2020LongformerTL}   & \underline{87.8}   & \underline{64.7}   & 74.3   & 71.6 & 98.47\% \\
Linformer~\cite{wang2020linformer}          & 73.1   & 48.7   & 62.8   & 60.1 & 97.96\% \\
Efficient Attention~\cite{shen2021efficient} & 84.4   & 62.6   & 73.7   & 70.1 & 97.98\% \\
\rowcolor{lightgray!40}
\textbf{\modelname (Ours)}                  & \textbf{91.8} & \textbf{70.7} & \textbf{79.1} & \underline{75.5} & 98.01\% \\
\bottomrule
\end{tabular}%
}
\caption{Performance comparison on CIFAR-10, CIFAR-100, Tiny-ImageNet, and ImageNet datasets with all methods integrated in ViT-B, highlighting the effect of attention pruning across approaches.}
\label{tab:sota}
\end{table}

\smartparagraph{\modelname vs.\ Other Sparse Attention Mechanisms.}
Table~\ref{tab:sota} compares our proposed \modelname against representative sparse-attention baselines across image classification datasets. We include a standard ViT-B~\cite{deit} with full self-attention~\cite{attention} as a dense reference, alongside random pruning, Top-$k$ pruning \cite{gupta2021memoryefficient}, Sparse Transformer~\cite{SparseTransformer19}, Efficient Attention~\cite{shen2021efficient}, Longformer~\cite{Beltagy2020LongformerTL}, BigBird~\cite{zaheer2020big}, and Linformer~\cite{wang2020linformer}. For a fair comparison, all sparse variants are configured to operate at a similar attention cost of approximately $0.014$ GFLOPs, corresponding to pruning ratios close to $98\%$, whereas the dense ViT-B baseline requires $0.72$ GFLOPs. Despite this extreme sparsity, \modelname consistently delivers strong performance across datasets. It substantially improves over the ViT-B baseline using full attention on C10, C100, and Tiny-IN, while remaining competitive on the more challenging IN-1K benchmark. In particular, although a small accuracy gap remains on IN-1K, \modelname achieves this performance while explicitly evaluating only about $2\%$ of token-to-token interactions, highlighting a favorable accuracy--efficiency trade-off at large scale. These results indicate that a significant portion of dense self-attention in ViTs is redundant, and that carefully designed sparse patterns can preserve, or even enhance, recognition accuracy under aggressive pruning. 

To better understand the impact of sparsity patterns in this high-pruning regime, we outline how different methods allocate their limited budget of token interactions: In Table~\ref{tab:sota}, \textit{BigBird} and \textit{Longformer} denote adaptations of the sparse attention schemes introduced in~\cite{zaheer2020big,Beltagy2020LongformerTL}, respectively, applied to ViTs under comparable pruning ratios. Since all sparse models operate at similar computational cost, observed performance differences primarily arise from how the remaining interactions are structured. Random pruning removes a fixed fraction of random attention entries without imposing any spatial or semantic structure. Although its computational cost matches that of structured methods, its accuracy is markedly lower across all datasets, including IN-1K, demonstrating that unstructured sparsity fails to preserve critical token relationships. Top-$k$ pruning, which retains the $k$ largest attention scores per query token, performs better than random masking, confirming that saliency-based criteria are beneficial under sparsity. However, Top-$k$ pruning underperforms structured approaches such as BigBird, Longformer and  \modelname on C10 and C100 by $5\%$--$10\%$, with a reduced performance gap on the mid-sized Tiny-IN dataset. We observe that while top-$k$ attention exceeds the performance of most structured sparse patterns on the large-scale IN-1K dataset, \modelname is the only mechanism that still outperforms top-$k$ for a comparable pruning ratio and is the only efficient attention mechanism that comes close to (by $0.4\%$ top-$1$ accuracy) the performance of a ViT-B trained with full attention. 

We conjecture that a main reason for the strong performance of \modelname is its head-specific, complementary masks, which enable distinct heads to capture local or global information in diverse and non-redundant feature representations. This phenomenon is further studied in Section \ref{sec:headdiversity} and Table~\ref{tab:diversity}.

\begin{table}[t]
\centering
\small
\setlength{\tabcolsep}{6pt}
\renewcommand{\arraystretch}{1.05}
\resizebox{\linewidth}{!}{%
\begin{tabular}{lccccc}
\toprule
\multirow{2}{*}{\textbf{ViT Variants}} &
\multicolumn{4}{c}{\textbf{Top-1 Accuracy (\%) $\uparrow$}} &
\multirow{2}{*}{\makecell[c]{\textbf{Pruning}\\\textbf{Ratio $\uparrow$}}} \\
\cmidrule(lr){2-5}&
C10 &
C100 &
Tiny-IN &
IN-1K &  \\
\midrule
ViT-B~\cite{deit}                       & 84.2 & 59.4 & 75.2 & \textbf{75.9} & 0\%  \\
\rowcolor{lightgray!40}
\textbf{+ \modelname}    & \textbf{91.8} & \textbf{70.7} & \textbf{79.1} & 75.5 & \textbf{98.0\%}  \\
\addlinespace[2pt]
UPop~\cite{shi2023upop}                 & 76.3 & 52.1   & 71.6   & \textbf{73.9} & 0\%  \\
\rowcolor{lightgray!40}
\textbf{+ \modelname}    & \textbf{82.8} & \textbf{56.8}   & \textbf{74.8}   & 72.0 & \textbf{97.1\%}  \\
\addlinespace[2pt]
iFormer~\cite{zheng2025iformer_iclr}    & \textbf{92.7} & \textbf{73.3}   & \textbf{88.4}   & \textbf{77.9} & 0\%  \\
\rowcolor{lightgray!40}
\textbf{+ \modelname}    & 92.2 & 72.7   & 88.1   & 77.1 & \textbf{94.4\%}  \\
\addlinespace[2pt]
Swin-B~\cite{liu2021swin}               & 81.2 & 60.7 & \textbf{82.8}   & \textbf{79.6}    & 0\%   \\
\rowcolor{lightgray!40}
\textbf{+ \modelname}    & \textbf{82.4} & \textbf{61.0} & 81.8   & 78.3    & \textbf{95.4\%}  \\
\addlinespace[2pt]
ConViT-B~\cite{d2021convit}             & 90.8 & 66.8 & \textbf{82.9}   & \textbf{82.1}    & 0\%   \\
\rowcolor{lightgray!40}
\textbf{+ \modelname}    & \textbf{90.9} & \textbf{67.5} & 82.8   & 80.1    & \textbf{96.6\%}  \\
\bottomrule
\end{tabular}%
}
\caption{Performance of \modelname integrated in various ViT variants, achieving 94–98\% attention pruning while maintaining competitive Top-1 accuracy.}
\label{tab:backbones}
\end{table}

\smartparagraph{\modelname Integrated into Various ViT Variants.}
In Table~\ref{tab:backbones}, we investigate the effect of integrating \modelname into a variety of ViT variants: ViT-B~\cite{deit}, UPop~\cite{shi2023upop}, iFormer~\cite{zheng2025iformer_iclr}, Swin-B~\cite{liu2021swin}, and ConViT-B~\cite{d2021convit} (see App. C-A 
for further details about the setup). For ViT-B and UPop, \modelname consistently yields higher accuracy on C10, C100, and Tiny-IN compared to their respective dense counterparts, while maintaining comparable performance on IN-1K for ViT-B and a slightly lower IN-1K accuracy for UPop. In both cases, the attention GFLOPs are reduced to only $2.0\%$ (ViT-B) and $2.9\%$ (UPop) of the dense baseline, demonstrating that substantial computational savings are possible even when the dense model is already well tuned.
For iFormer, which already embeds a strong inductive bias via single-head attention, incorporating \modelname results in an accuracy that remains close to the original backbone across all datasets, while reducing attention FLOPs to $5.6\%$ of the original cost, leading to a reasonable efficiency-effectiveness trade-off.
For hierarchical backbones such as Swin-B and ConViT-B, which introduce local inductive biases through token merging or convolutions, \modelname yields accuracy that is broadly comparable to the base models. For Swin-B, the performance with \modelname stays within a narrow band of the dense model on all four datasets, while reducing attention GFLOPs to $4.6\%$ of the original. For ConViT-B, \modelname maintains similar accuracy on C10, C100, and Tiny-IN, and a slightly lower performance on IN-1K, but with attention GFLOPs reduced to $3.4\%$ of the dense baseline. Overall, these results indicate that \modelname is highly compatible with a range of Transformer architectures, delivering considerable computational savings together with improved performance for architectures with limited inductive bias when trained on small-sized datasets, and no or limited accuracy degradation when trained on mid-sized or large datasets compared to full attention variants, depending on the backbone and dataset.

\begin{table}[ht]
\centering
\small
\setlength{\tabcolsep}{5.9pt}
\caption{Top-$1$ accuracy of ViT variants trained with sparse attention mechanisms, configured to compute only $\approx2\%$ of self-attention entries.}
\label{tab:pruning_results}
\resizebox{\textwidth}{!}{%
\begin{tabular}{lcccccccc}
\toprule
\multirow{2}{*}{\textbf{Method}} &
\multicolumn{2}{c}{\textbf{ViT-B}} &
\multicolumn{2}{c}{\textbf{UPop}} &
\multicolumn{2}{c}{\textbf{ConViT-B}} &
\multirow{2}{*}{\shortstack{\textbf{Pruning}\\\textbf{Ratio $\uparrow$}}} \\
\cmidrule(lr){2-3} \cmidrule(lr){4-5} \cmidrule(lr){6-7}
 & C10 & C100 & C10 & C100 & C10 & C100 & \\
\midrule
Full Attention~\cite{attention} & 84.2 & 59.4 & 76.3 & \underline{52.1} & \underline{90.8} & \underline{66.8} & 0\% \\
\midrule
Random Attention~\cite{zaheer2020big} & 80.7 & 56.5 & 72.3 & 46.8 & 89.1 & 64.5 & 98.52\% \\
Top-$k$ Attention~\cite{gupta2021memoryefficient} & 81.1 & 57.1 & 71.5 & 36.8 & 89.6 & 64.4 & 98.48\% \\
BigBird~\cite{zaheer2020big} & 86.8 & 63.4 & 79.1 & 37.7 & 89.8 & 64.7 & 97.96\% \\
Longformer~\cite{Beltagy2020LongformerTL} & \underline{87.8} & \underline{64.7} & \underline{79.6} & 39.4 & 89.2 & 64.3 & 98.47\% \\
\rowcolor{gray!15}
\textbf{\modelname~(Ours)} & \textbf{89.5} & \textbf{64.9} & \textbf{82.8} & \textbf{56.8} & \textbf{90.9} & \textbf{67.5} & 98.01\% \\
\bottomrule
\end{tabular}%
}
\end{table}

\smartparagraph{Sparse Attention Mechanisms Across ViT Variants.}
Table~\ref{tab:pruning_results} reports the performance of \modelname and selected sparse attention methods across multiple ViT-style backbones (ViT-B, UPop and ConViT-B) trained on C10 and C100 under a high sparsity setting where only $2\%$ of self-attention entries are retained. We observe that among the considered sparse attention methods, only \modelname consistently improves accuracy over the full attention baseline. This trend persists across all evaluated backbones, indicating that \modelname's advantages robustly generalize across Transformer architectures. 


\section{\modelname~in Other Visual Domains} \label{sec:fibottention:othervisual}
We demonstrate the versatility of \modelname by integrating it into ViTs designed for solving visual perception tasks beyond image classification.

\subsection{Video Action Classification}

\noindent\textbf{Datasets.} We evaluate and report top-1 action classification accuracy for \modelname~using three action recognition datasets: Toyota Smarthome~\cite{smarthome} and Northwestern-UCLA Multiview Activity 3D~(NUCLA)~\cite{nucla}. The Toyota Smarthome dataset comprises $\sim$$16\text{K}$ videos across $31$ classes. Here, we adhere to the cross-subject (CS) and cross-view (CV2) protocols. The NUCLA dataset consists of $\sim$$1.2\text{K}$ video clips with subjects performing 10 different action classes and we use the cross-subject (CS) protocol.

\noindent\textbf{Training.} For this experiment, we employ the divided-space-time attention variant of TimeSformer \cite{timesformer} for action classification. We integrate \modelname\ into the spatial attention module of TimeSformer, considering that the temporal attention module already processes dense attention across the same patch in contiguous frames.
For the implementation of \modelname, we use both its Wythoff and Modified Wythoff variants. The hyperparameters for \modelname~are set to $\wmin = 1$ and $\wmax = 196$.
Given that the TimeSformer architecture fundamentally resembles a ViT-B with additional attentional modules, it is initialized with IN-1K pre-trained weights. All video models are trained with a batch size of 32 for 15 epochs.
For the Toyota Smarthome dataset, we process video clips of size $8\times224\times224$ with a sampling rate of 1/32, while for NUCLA, we use video clips of size $16\times224\times224$ with a sampling rate of 1/4.

\noindent\textbf{Results.} In Table~\ref{tab:video}, we compare the action classification results using TimeSformer and other attention mechanisms (BigBird, and {\modelname}) integrated within TimeSformer. We observe that \modelname~with Modified Wythoff instantiation outperforms all baselines on the Smarthome and NUCLA protocols, utilizing a masking percentage of 94\%. \modelname\ with Modified Wythoff facilitates increased local interactions among query-key pairs compared to the original Wythoff sequences, albeit at the expense of a reduced masking ratio (by 1.5\%). The Modified Wythoff proves essential in our video experiments, where capturing the temporal evolution of local patches is critical for learning discriminative spatiotemporal representations.

\begin{table}[b]
  \centering
  \scriptsize
  \setlength{\tabcolsep}{3.5pt}
  \resizebox{\linewidth}{!}{%
    \begin{tabular}{lccc}
      \toprule
      \textbf{Method} & \multicolumn{2}{c}{\textbf{SmartHome}~\cite{smarthome}} & \textbf{NUCLA}~\cite{nucla} \\
                      & \textbf{CS} & \textbf{CV2} & \textbf{CS} \\
      \midrule
      TimeSformer~\cite{timesformer} & 52.2 & 36.6 & 32.9 \\
      \quad + BigBird~\cite{zaheer2020big}                 & 51.4 & 40.1 & 50.9 \\
      \rowcolor{gray!15}
      \quad + \modelname~(Wythoff)            & 55.6 & 38.6 & 49.3 \\
      \rowcolor{gray!15}
      \quad + \modelname~(Modified Wythoff)   & \textbf{57.1} & \textbf{42.3} & \textbf{59.6} \\
      \bottomrule
    \end{tabular}%
  }
  \caption{\modelname\ Top-1 accuracy on Smarthome and NUCLA.}
  \label{tab:video}
\end{table}


\subsection{Robot Learning}
For robotics experiments, we assess the performance of \modelname~for behavioral cloning~\cite{florence2021implicit} in which we aim to learn a robot policy by training a model on state-action pairs obtained from human examples.

\noindent\textbf{Datasets.}
We evaluate three datasets: Can and Lift from Robomimic~\cite{robomimic2021}, and PushT from Implicit Behavioral Cloning~\cite{florence2021implicit}. In Lift, the robot must lift a cube to a specific height. In Can, the robot must move a can into a box. In PushT, the robot must align a T-shaped block with a T-shaped outline. We provide visuals of all three datasets in Figure~\ref{fig:robotics-dataset-visuals}.

\begin{figure}[t]
    \centering
    \includegraphics[width=0.8\textwidth]{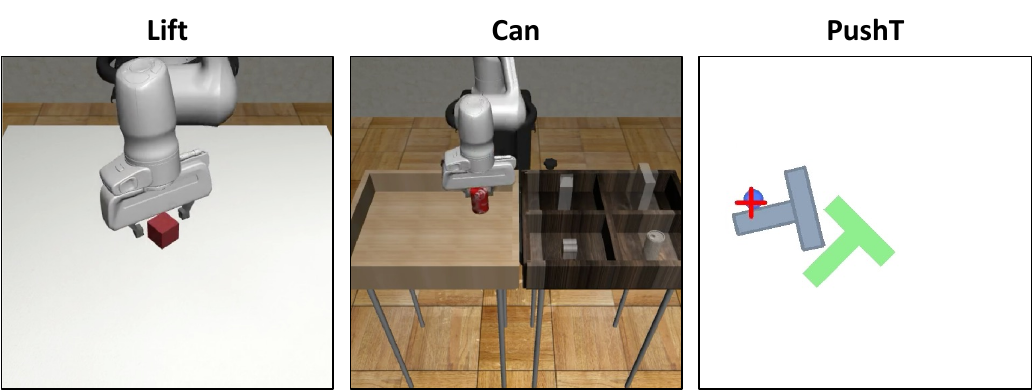}
    \caption{Sample frames from the datasets used in our robotics experiments.}
    \label{fig:robotics-dataset-visuals}
\end{figure}

\noindent\textbf{Training.}
Building upon the Crossway Diffusion~\cite{li2024crosswaydiffusion} framework, we modify the architecture by substituting the ResNet visual backbone with a ViT~\cite{dosovitskiy2020vit} and incorporating \modelname into standard ViT self-attention layers. We employ a batch size of 64 and utilize ViT-B with a patch size of 8 as the visual backbone. For all other hyperparameters, including the number of epochs, we follow~\cite{li2024crosswaydiffusion}.

\noindent\textbf{Results.}
We report the average task completion accuracy in Table~\ref{tab:robotics} and find that \modelname with Modified Wythoff instantiation leads to improvements over both the baseline ViT and ViT with
BigBird attention.

\begin{table}[h]
    \centering
    \footnotesize
    \setlength{\tabcolsep}{4.8pt}
    \resizebox{\linewidth}{!}{%
    \begin{tabular}{lc p{1cm} p{1cm} p{1cm}}
        \toprule
        \textbf{Visual Backbone} & \textbf{Lift} & \textbf{Can} & \textbf{PushT} \\
        \toprule
        ViT-B~\cite{deit} & 0.980 & 0.960 & 0.678 \\
        \quad+ BigBird~\cite{zaheer2020big} & 1.000 & 0.880 & 0.690 \\
        \rowcolor{lightgray!40}
        \quad+ \modelname~(Wythoff) & 0.820 & 0.940 & 0.630 \\
        \rowcolor{lightgray!40}
        \quad+ \modelname~(Modified Wythoff) & \textbf{1.000} & \textbf{0.960} & \textbf{0.720} \\
        \bottomrule
    \end{tabular}%
    }
    \caption{Performance of {\modelname} on behavioral cloning for robotics. The average task completion accuracy is reported.}\label{tab:robotics}
\end{table}

\section{Ablation Studies and Analytical Findings}\label{sec:further-ablation}

Beyond the experiments above, we perform a set of controlled ablations to analyze the effect of individual architectural and sparsity choices in \modelname. Unless explicitly stated, all experiments in this section use ViT-B as the visual backbone.

\subsection{Validation of Head Diversity}\label{sec:headdiversity}
One of the design principles behind \modelname~is that the resulting feature representations are more \emph{diverse across heads} than those of standard MHSA. To validate this claim, we consider the following analysis protocol: For ViT-B trained on C10 with different sparse attention mechanisms, we perform inference for single CIFAR-10 images $X$ and compute the last-layer feature matrices $Y_i = \sigma(A^{\Omega_i}_i) V_i$ for each head $i=1,\ldots,h$, where $A^{\Omega_i}$ is as in \eqref{eq:Omegamasking}. Then we measure the distances $\|Y_i - Y_j\|_F$ of feature representations across heads for each pair $(Y_i,Y_j)$, $1 \leq i < j \leq h$ and average such relative distances in the diversity metric
\[
\operatorname{diversity}(X) = \frac{2}{h(h-1)}
\sum_{i<j}\;
\frac{\,\left\lVert Y_i - Y_j \right\rVert_F}{
\left\lVert Y_i \right\rVert_F + \left\lVert Y_j \right\rVert_F },
\]
    which we report for many different input images $X$. Within this framework, higher values of $\operatorname{diversity}(X)$ indicate greater variability in information captured by different heads for a given input $X$. Considering $10^4$ input images $X$, we report aggregated information (minimum, maximum, quartiles and median) of the distribution of the head diversity metric $\operatorname{diversity}(X)$ in Table~\ref{tab:diversity} for the considered sparse attention methods. We observe that \modelname~consistently attains a significantly larger median head diversity than ViT-B and all other sparse attention mechanisms. Arguably, this is due to \modelname's deliberate use of complementary, head-specific sparsity patterns, whereas the sparse baselines apply the same mask across heads. The increased observed head diversity indicates that \modelname~captures richer and more complementary feature representations than other sparse variants, which we believe is key to its superior performance under constrained FLOPs (see Table~\ref{tab:sota}).
\begin{table}[ht]
\centering
\small
\setlength{\tabcolsep}{3.8pt}
\caption{Head diversity statistics (Frobenius distances across last-layer features) over $10^4$ CIFAR-10 images. \modelname yields substantially higher diversity across heads.}
\label{tab:diversity}
\resizebox{\linewidth}{!}{%
\begin{tabular}{lccccc}
\toprule
\textbf{Method} & \textbf{Min} & \textbf{Max} & \textbf{Median} & \textbf{Q1} & \textbf{Q3} \\
\midrule
ViT-B~\cite{deit}                  & 27.57 & 61.53 & 43.00 & 37.68 & 48.85 \\
+ Random Attention~\cite{zaheer2020big}                  & 13.68 & 47.74 & 29.24 & 22.08 & 35.31 \\
+ Sparse Transformer~\cite{SparseTransformer19} & 28.84 & 63.71 & 43.27 & 38.27 & 49.65 \\
+ Longformer~\cite{Beltagy2020LongformerTL}     & 34.49 & 65.27 & 49.26 & 44.52 & 54.29 \\
+ BigBird~\cite{zaheer2020big}                  & 34.15 & 72.08 & 51.86 & 45.67 & 58.55 \\
\rowcolor{gray!15}
\textbf{+ \modelname} & 41.63 & 75.95 & \textbf{57.34} & 51.98 & 63.22 \\
\bottomrule
\end{tabular}%
}
\end{table}


\begin{figure*}[t]
\floatsetup[table]{capposition=bottom}
\CenterFloatBoxes

\begin{floatrow}[3]

\ffigbox[0.31\textwidth]{
\caption{Inference costs of ViT-B, ViT-T, ConViT-B (Vanilla vs.\ \modelname).
} \label{fig:total_flops}
}{
\centering
\includegraphics[width=\linewidth]{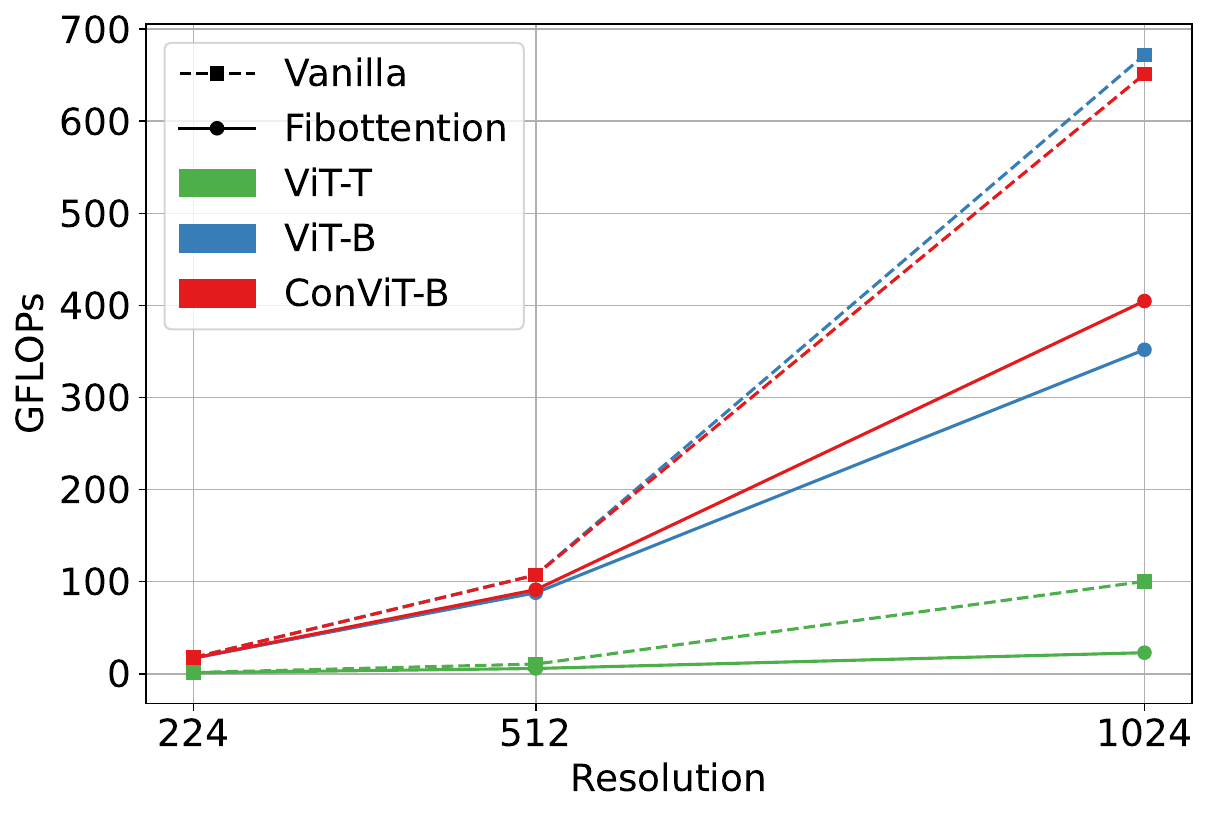}
}

\ffigbox[0.31\textwidth]{
\caption{Test accuracy of ViT-B models for corrupted datasets CIFAR-10 C and CIFAR-100 C, trained with and without \modelname \label{fig:corruption}
}}
{
\centering
\includegraphics[width=\linewidth]{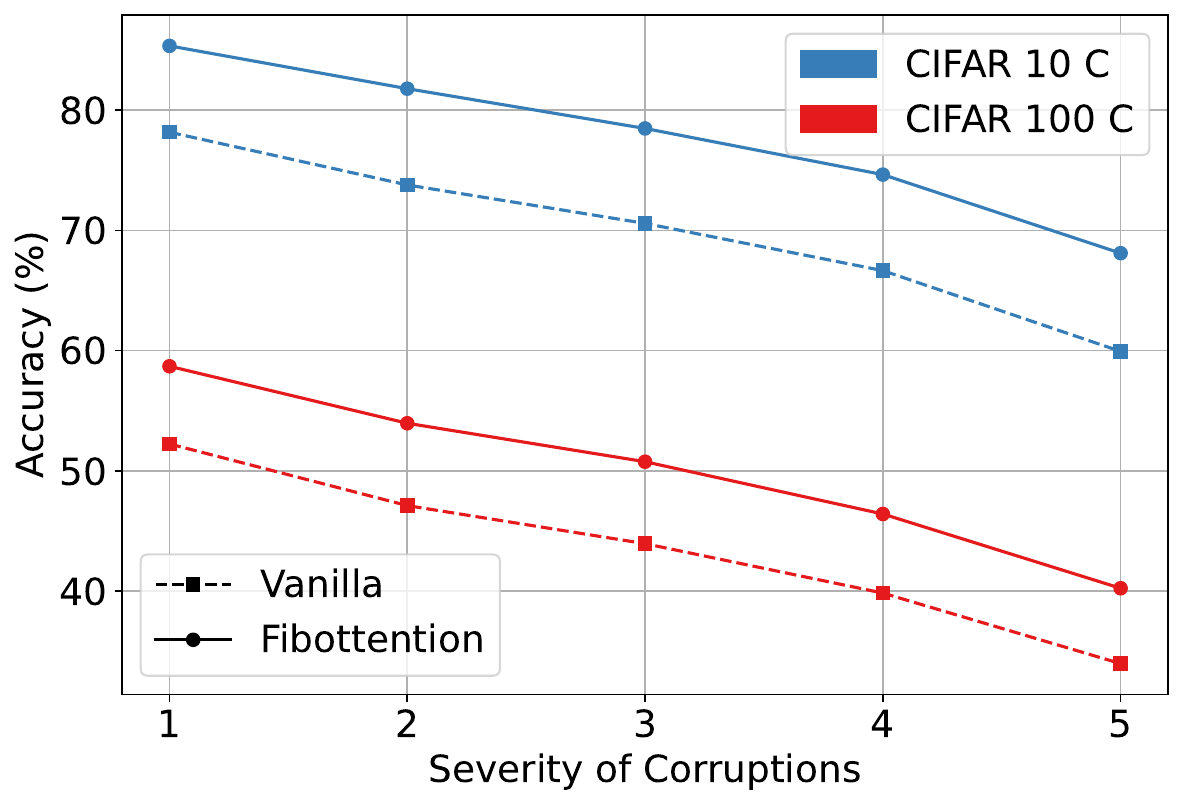}
}

\killfloatstyle
\ttabbox{
    \begin{tabular}{lcc}
        \toprule
        \textbf{Sequences} & \textbf{C10}~\cite{cifar} & \textbf{C100}~\cite{cifar} \\
        \midrule
        $(2^n)_{n\in\mathbb{N}}$              & 86.1 & 61.6 \\
        $(3^n)_{n\in\mathbb{N}}$              & 85.3 & 60.2 \\
        $(n^2)_{n\in\mathbb{N}}$              & 85.8 & 61.5 \\
        $(n^3)_{n\in\mathbb{N}}$              & 84.6 & 59.1 \\
        \rowcolor{lightgray!40}
        $\Fib(1,1)$                           & \textbf{87.3} & \textbf{63.2} \\
        \bottomrule
    \end{tabular}
}
{\caption{Effect of different dilation sequences $(f_n)_{n\in\mathbb{N}}$ on CIFAR-10~\cite{cifar} and CIFAR-100~\cite{cifar}, with a fixed window size $w_i = N/3$. \label{tab:sequences}}}

\end{floatrow}
\end{figure*}

\subsection{Validation of Inductive Bias}\label{sec:inductivebias}
To validate the inductive bias in \modelname, we examine its “inside-to-inside” attention patterns within specific regions of input images. \modelname's structured pruning mechanism is designed to promote focused attention within predefined boundaries. We hypothesize that this mechanism leads to higher average attention scores and lower variance within localized areas, reflecting a stronger inductive bias toward these regions. To test this hypothesis, we calculate the variance of attention scores within object-localized regions by aggregating and analyzing attention values across designated patches. This approach enables a direct comparison of attention distribution patterns between a \modelname-trained ViT-B model and a corresponding model with full attention.
Results averaged over 100 images show that \modelname~achieves a significantly lower variance of attention scores within localized areas (\(2.33 \times 10^{-5}\)) compared to full attention ViT-B (\(8.76 \times 10^{-5}\)). This reduced variance indicates that \modelname~maintains more concentrated and consistent attention within target regions.


\subsection{Computational Complexity}
\label{subsec:complexity} 
We recall that the inference FLOPs that accrue at a \modelname-modified self-attention layer scale with $O(N \log N)$ if $N$ is the number of token/patches (cf. App. A-C). 
To illustrate the implications of this in practice, we compare the projected inference cost per input of ViT-B~\cite{dosovitskiy2020vit}, ViT-T~\cite{wu2022tinyvit}, and ConViT~\cite{d2021convit} in Figure~\ref{fig:total_flops}, considering their dense (i.e., \emph{vanilla}) self-attention alongside their \modelname variants. It is well known \cite{dosovitskiy2020vit,Kaplan2020scaling} that for smaller resolutions, a substantial part of the computation comes from non-MHSA components of the Transformer architectures, such as the feed-forward/MLP blocks \cite{Kaplan2020scaling}, which is why the total inference FLOPs improvement of \modelname (or any sparse attention method, for that matter) is rather limited at a $224\times224$ resolution. On the other hand, we see in Figure~\ref{fig:total_flops} that 
as resolution, and thus, $N$ increases, self-attention computation becomes a much larger fraction of the total computational load so that \modelname is able to reduce FLOP counts by up to $48\%$. In ConViT~\cite{d2021convit}, which contains 10 GPSA and 2 MHSA modules, we apply \modelname only to the corresponding MHSA modules.

\subsection{Robustness of \modelname}
\label{subsec:robustness}
To understand \modelname's robustness with respect to distribution shifts, we evaluate its impact on  predictive performance for the C10 and C100 {\em corrupted}~\cite{hendrycks2019robustness} datasets. These datasets contain test images from the original datasets, corrupted by 19 common image corruptions with severity levels ranging between 1--5 (see Appendix E-C 
for more details). 
For the evaluation, we use models pre-trained on default C10 and C100 data, and evaluate them on the corrupted datasets across varying severity. Figure~\ref{fig:corruption} shows that \modelname~ models perform consistently better than their dense MHSA counterparts across all corruption severity levels.


\subsection{Choice of Dilation Sequences}
\label{sec:selecting-dilation-patterns}

Table~\ref{tab:sequences} compares several choices of dilation sequences $(f_n)_n$ used to construct the diagonal offsets in the sparsity sets $\Omega_i$, while keeping a fixed window size of $w_i = \wmin = \wmax = N/3$ for all heads. We include polynomial growth sequences such as $(n^2)_{n \in \mathbb{N}}$ and $(n^3)_{n \in \mathbb{N}}$, as well as exponential sequences $(2^n)_{n \in \mathbb{N}}$ and $(3^n)_{n \in \mathbb{N}}$ that resemble dilated patterns commonly used in prior work~\cite{li2019enhancing}. 
Among all options considered, the standard Fibonacci sequence $\Fib(1,1)$ yields the highest accuracy on both CIFAR-10 and CIFAR-100. Intuitively, Fibonacci growth is slow enough that many small offsets remain available, which reinforces local interactions, but still introduces a few longer-range diagonals. This balance appears to be better suited to image data than the more aggressive growth patterns of the alternative sequences.


\subsection{Impact of Dilation Across Heads}
\label{subsec:impact-dilation}

In Figure~\ref{fig:unified}(a) we examine the role of varying dilation sequences across heads. We compare two settings that both use sequences of the form $(f_n)_{n\in\mathbb{N}} = (c \cdot n)_{n\in\mathbb{N}}$: a \textit{fixed} configuration in which all heads share the same dilation sequence, and a \textit{variable} configuration in which each head receives a shifted version of the sequence, so that offsets differ from head to head.
At comparable masking ratios, allowing the dilation patterns to vary across heads consistently improves performance. For instance, at $\wmin = 2$, the masking ratio increases from $72.8\%$ for fixed sequences to $85.5\%$ for variable sequences, while resulting in improved accuracy. This confirms that head-wise diversity in the sparsity pattern is an effective way to maintain performance even for increased attention sparsity levels.

\begin{figure*}[t]
    \centering

    \subfloat[]{%
        \includegraphics[width=0.32\linewidth]{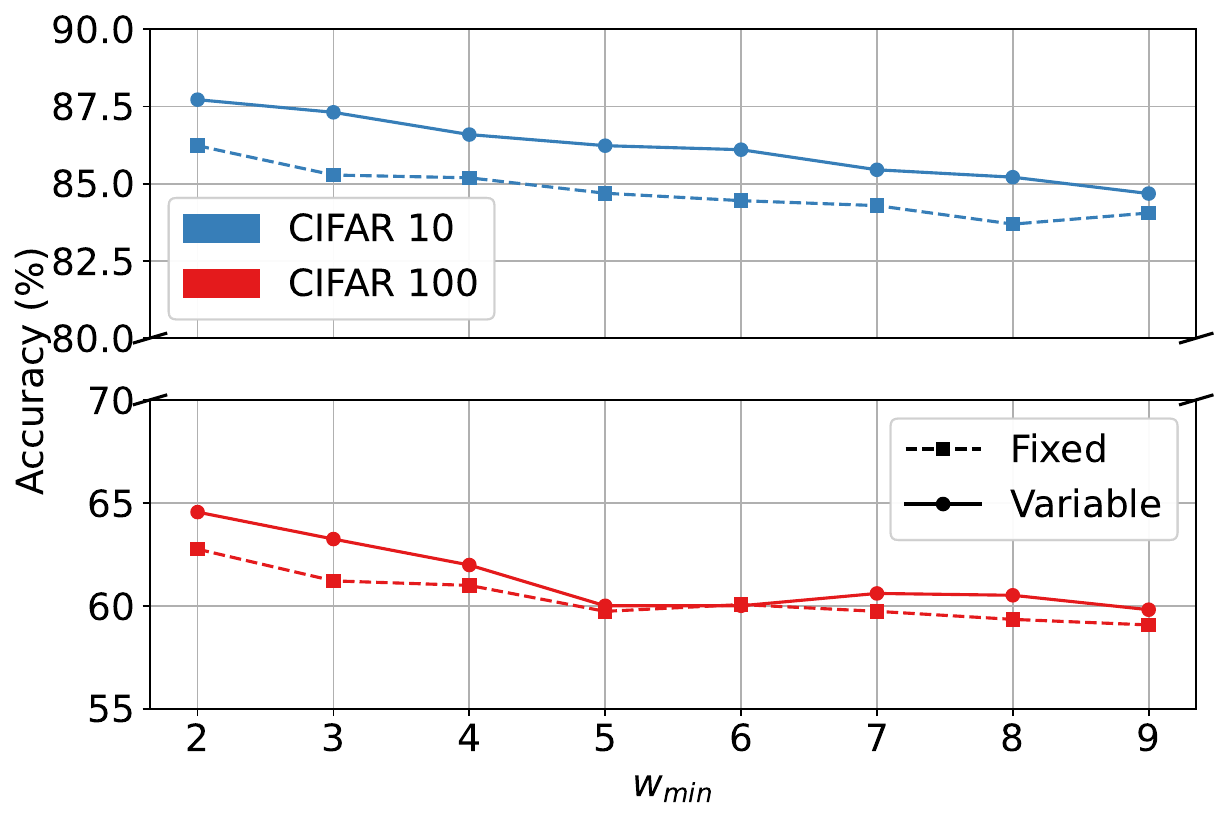}%
        \label{fig:dilated-fixed-gap}%
    }\hfill
    \subfloat[]{%
        \includegraphics[width=0.32\linewidth]{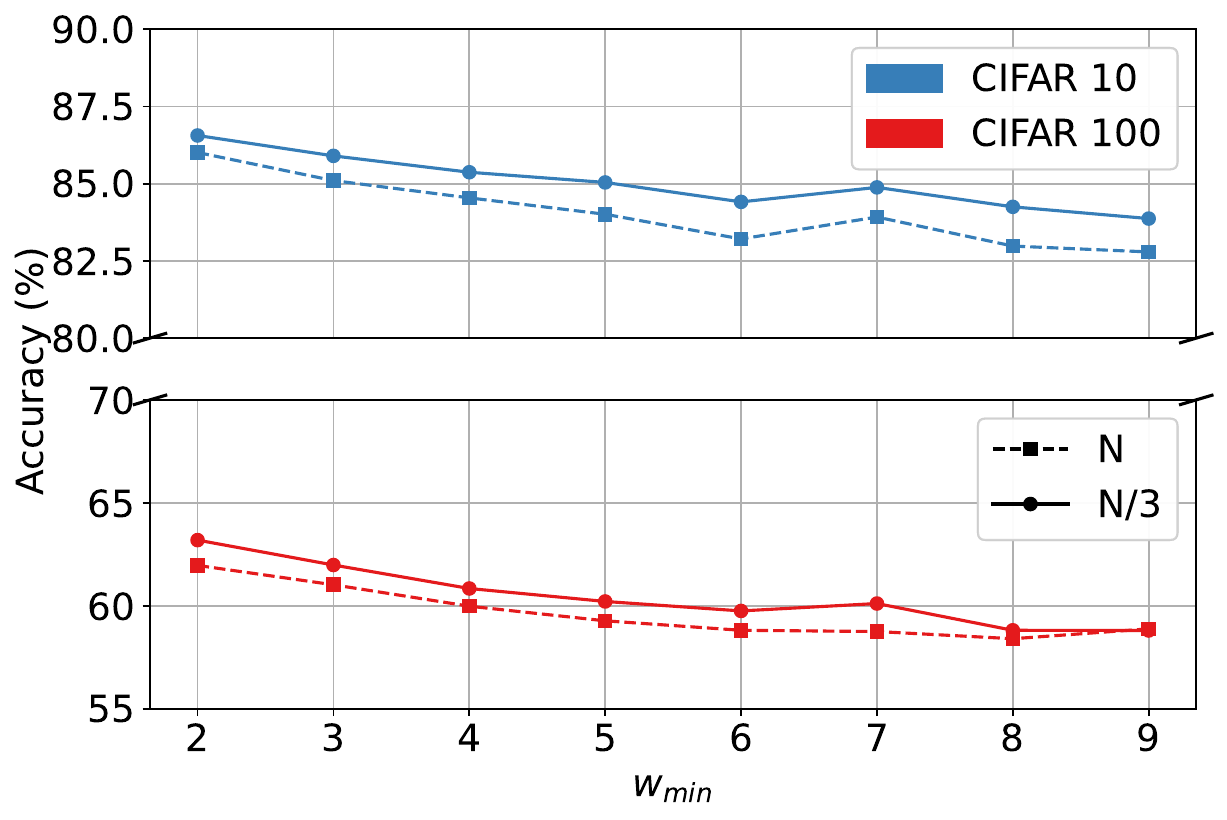}%
        \label{fig:fixed-offset}%
    }\hfill
    \subfloat[]{%
        \includegraphics[width=0.32\linewidth]{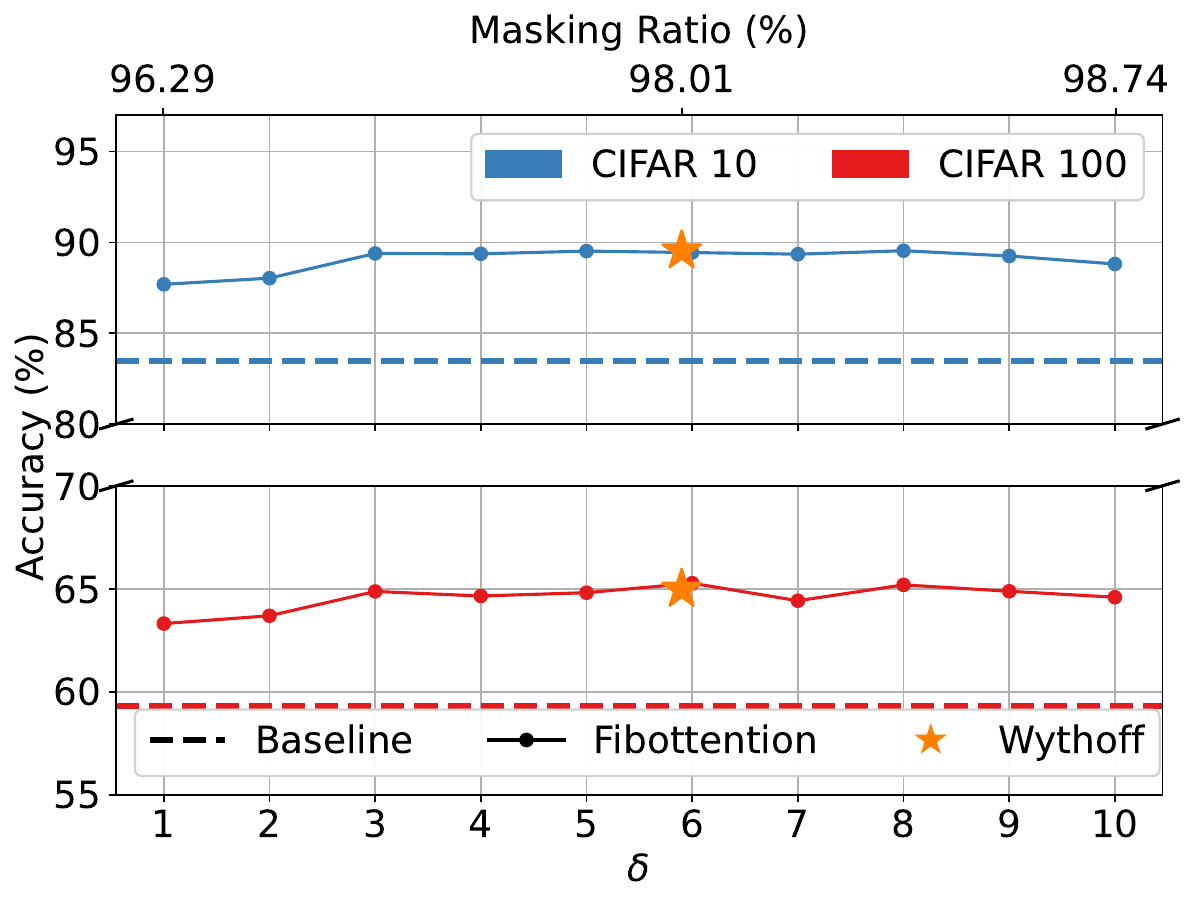}%
        \label{fig:fixed-index}%
    }

    \caption{Ablation study of (a) impact of dilation with sequences $(f_n)_{n\in\N}=(c n)_{n\in\N}$ fixed and variable across heads where $w_i=5h_i$; (b) choice of $\wmax$ with sequences $(f_n)_{n\in\N}=\Fib(\wmin,2\wmin)$ where $\wmax=N$ and $\wmax=N/3$ fixed across all heads; and (c) variable dilation sequences $\Fib(i+\delta,i+\delta)$ for the $i$-th head, $i\in\{1,\ldots,12\}$, with varying $\delta$, vs.\ Wythoff.}
    \label{fig:unified}
\end{figure*}


\subsection{Choice of Window Size Limit \texorpdfstring{$\wmax$}{wmax}}
\label{sec:choice-of-wmax}

Figure~\ref{fig:unified}(b) compares two ways of setting the maximum window size $\wmax$ when using Fibonacci sequences of the form $(f_n)_{n\in\mathbb{N}} = \Fib(\wmin, 2 \wmin)$: a fully global setting where $\wmax = N$, and a more local setting where $\wmax = N/3$.  
We observe that $\wmax = N/3$ consistently leads to better accuracy than $\wmax = N$. In image classification most of the discriminative structure tends to be spatially localized, so interactions with very distant tokens are often less informative and can even dilute object-level features. Restricting $\wmax$ to roughly a third of the sequence length encourages attention to focus on more relevant neighborhoods. In addition to this, we study in Appendix E-B 
and Table XIII 
the special case of coinciding maximum and minimum window sizes, i.e., $\wmax = \wmin$, resulting in a shared fixed window size $w_i$ across all heads. These experiments likewise suggest that moderate window sizes can be beneficial for model accuracy.

\subsection{Why Wythoff?}
\label{sec:why-wythoff}

The generalized Fibonacci sequences appearing in the Wythoff array have the special property that each positive integer appears in exactly one row, which in our context translates into minimal overlap of diagonal offsets across heads. To investigate this question, we compare in Figure~\ref{fig:unified}(c) \modelname\ configured with Wythoff-based generalized Fibonacci sequences against several families of head-specific Fibonacci sequences that use additional offset hyperparameters. For all models in this comparison, we fix $\wmin = 5$ and $\wmax = N/3$.
We observe that the Wythoff configuration attains the highest accuracy, while having no extra hyperparameters beyond $(\wmin, \wmax)$. This suggests that the highly complementary and non-overlapping head masks of \modelname's Wythoff construction are indeed a driver of the strong model performance of Transformers using \modelname instead of MHSA in practice.

\section{Conclusion} \label{sec:conclusion}

We introduce \modelname,~an efficient, robust, $O(N \log N)$ sparse attention mechanism with a fixed sparsity pattern that diversifies attention computation across heads through Fibonacci dilation sequences chosen from the Wythoff array. We implement \modelname~in conjunction with multiple state-of-the-art Transformer architectures curated for visual representation learning. Empirically, \modelname~outperforms the baselines on small-scale and mid-scale datasets and achieves comparable performance on large-scale datasets utilizing only 2-6\% of token interactions in the MHSA across three diverse visual tasks. We envision that the next generation of Transformers~\cite{geminiteam2024gemini} processing inputs with billions of tokens may benefit from such optimized architectures. It remains to future work to generalize this idea to application domains that involve causal attention, such as Transformer models for natural language processing tasks or time series analysis. 

\section*{Acknowledgements}

This work is supported in part by the National Science Foundation under Grant IIS-2245652 and the Chateaubriand Fellowship of the Office for Science and Technology of the Embassy of France in the United States. Aritra Dutta and Subhajit Maity are partially supported by the Florida Department of Health Grant AWD00007072 and the National Science Foundation Grant 2321986. The authors are grateful to Ahmed Helmy for providing essential computational resources.%

\bibliographystyle{IEEEtran}
\bibliography{main}

\appendices
\setcounter{section}{0}

\section{Computational Complexity of \modelname}\label{sec:complexity}
In this section, we analyze the time complexity of using \modelname within a forward pass of a Transformer architecture. To this end, we first recall an explicit formula for the $n$-th sequence element of a generalized Fibonacci sequence in \Cref{sec:binet:formula}. We then use this to bound the sparsity of the evaluated attention matrix in \Cref{sec:bound:attentionsparsty} for a given attention head. Finally, we use this result to bound the total computational complexity of \modelname in \Cref{sec:bound:Fibottention:complexity}.

\subsection{Generalized Binet's Formula} \label{sec:binet:formula}
We state a well-known generalization of Binet's formula \cite{Koshy2019fibonacci} to generalized Fibonacci sequences, which states that the $n$-th sequence element $f_n$ of the standard Fibonacci sequence $\Fib(1,1)$ (i.e., the sequence $(f_n)_{n} = (1,1,2,3,5,8,12,\cdots)$)
satisfies
\[
f_n = \frac{1}{\sqrt{5}}\left(\phi^{n-1} - \psi^{n-1} \right),
\]
where $\phi = \frac{1+\sqrt{5}}{2}$ and $\psi = \frac{1-\sqrt{5}}{2}$ are the two solutions of the quadratic equation defining the golden ratio. This provides an explicit, non-recursive characterization of the $n$-th sequence element.

It turns out that this formula can be generalized to generalized Fibonacci sequences $\Fib(a,b)$ as defined in \Cref{sec:fibottention}, which follow the same linear recurrence relation \cref{eq:recurrence}, but which start at $f_1 =a $ and $f_2 = b$ for $a,b \in \N$, see \Cref{lemma:generalized:binet}.
\begin{lemma}[Generalized Binet's Formula \cite{Koshy2019fibonacci}] \label{lemma:generalized:binet}
    If $\text{Fib}(a,b) = (f_n)_{n \in \N}$ is the generalized Fibonacci sequence with initial values $f_1 = a$ and $f_2 = b$, then it holds that
    \begin{equation} \label{eq:gen:Binet:1}
    f_n = \frac{a - (b-a) \psi}{\sqrt{5}} \phi^{n} + \frac{(b-a) \phi - a}{\sqrt{5}} \psi^{n} 
    \end{equation}
    for each $n \geq 1$ and 
    \begin{equation} \label{eq:gen:Binet:2}
    f_n = \frac{b - a \psi}{\sqrt{5}} \phi^{n-1} + \frac{a \phi - b}{\sqrt{5}} \psi^{n-1} 
    \end{equation}
    for each $n \geq 2$, where $\phi = (1+\sqrt{5})/2$ and $\psi= (1- \sqrt{5})/2$.
\end{lemma}
The proof of Lemma \ref{lemma:generalized:binet} is standard and follows the proof of the conventional Binet's formula. We provide it below for completeness.
\begin{proof}[{Proof of Lemma \ref{lemma:generalized:binet}}]
We note that the defining linear recurrence relation
\[
f_{n+1} = f_n + f_{n-1}.
\]
of the generalized Fibonacci sequence is homogeneous and has the characteristic equation
\begin{equation} \label{eq:charac:eq}
x^2 - x - 1 = 0,
\end{equation}
which has the roots $\phi$ and $\psi$ as defined in Lemma \ref{lemma:generalized:binet}. Due to the homogeneity of the linear recurrence relation, it follows that
\[
f_n = A \phi^n + B \psi^n,
\]
where \( A \) and \( B \) are constants to be determined from the initial conditions $f_1 = a$ and $f_2 =b$. In particular, from  \( f_1 = a \), we obtain that
\[
a = A \phi + B \psi
\]
and furthermore, from \( f_2 = b \), we see that
\[
b = A \phi^2 + B \psi^2 = A (\phi + 1) + B( \psi + 1),
\]
where we used the characteristic equation \eqref{eq:charac:eq} for solutions $\phi$ and $\psi$. For this, we obtain the system of equations
\[
(A + B) + a  = b,
\]
\[
A \phi + B \psi = a.
\]
From the first equation, we obtain \( B = (b - a) - A \). Substituting this into the second equation, this results in
\[
A \phi + ((b - a) - A) \psi = a,
\]
which can be rearranged to
\[
A (\phi - \psi) = a - (b-a) \psi
\]
and finally
\[
A = \frac{a - (b-a) \psi}{\phi - \psi} = \frac{a - (b-a) \psi}{\sqrt{5}}
\]
using that $\phi - \psi = \sqrt{5}$. 
For $B$, we obtain
\[
B = b - a - \frac{a - (b-a) \psi}{\phi - \psi} = \frac{(b-a) \phi - a}{\phi-\psi} = \frac{(b-a) \phi - a}{\sqrt{5}}.
\]
This implies equation \eqref{eq:gen:Binet:1}. Finally, we observe that
\[
\frac{a - (b-a) \psi}{\sqrt{5}} \phi = \frac{(b-a)+a\phi}{\sqrt{5}} = \frac{b + a (\phi - 1)}{\sqrt{5}} =\frac{b - a \psi}{\sqrt{5}}
\]
and also, that
\[
\frac{(b-a) \phi - a}{\sqrt{5}} \psi = \frac{(a-b)- a\psi}{\sqrt{5}} = \frac{a(1-\psi) - b}{\sqrt{5}} = \frac{a \phi - b}{\sqrt{5}},
\]
which precisely implies~\eqref{eq:gen:Binet:2}.
\end{proof}

\subsection{Sparsity of Attention Matrix with Fibonacci Dilation Sequence}
\label{sec:bound:attentionsparsty}
Now we are set to provide the head-wise computational overhead of the standard and modified variant of \modelname. 
\begin{lemma}\label{lemma:complexity}
Let $N$ be the number of tokens in a multi-head self-attention block.
    If $(f_n)_n = \Fib(a,b)$ is used as a dilation sequence for $a , b \in \N$ to create the attention support set $\Omega_{w}^{\Fib(a,b)}$ as in \cref{def:Omegai:Fibottention} for window size $w \leq N$, $a < b \leq w$, then the masked attention matrix $\cA^{\Omega_{w}^{\Fib(a,b)}}$ can be computed by evaluating at most 
    \[
    |\Omega_{w}^{\Fib(a,b)}| \leq 2N \left(\frac{\log(\sqrt{5}w+|a\phi - b|) - \log(b-a \psi)}{\log\phi}+1\right)  
    \] dot products between query and key vectors, where $\phi = (1+\sqrt{5})/2$ is the golden ratio.
\end{lemma}
\begin{proof}
Let $(f_1,\cdots,f_D)$ be the Fibonacci indices for one mask. We have $f_{k+1} = f_k + f_{k-1}$. Let $f_1 = a_i$ and $f_2 = b_i$. Let $f_j$ be the index of the diagonal. Thus the total number of inner products to be computed for the diagonal $f_j$ is given by $2(N-f_j)$ considering the symmetric distribution of diagonals at which the self-attention matrix is evaluated. For a total of $D$ diagonals indicated by sequence elements $(f_1,\cdots,f_D)$ of $\Fib(a,b)$ such that $f_D \leq w$, but $f_{D+1} > w$, we obtain a sparsity pattern whose size is bounded by
\[
|\Omega_{w}^{\Fib(a,b)}| = \sum_{j=1}^D 2(N-f_j).
\]
Using the identity
\begin{equation} \label{eq:Fib:sumidentity}
    \sum_{j=1}^D f_j + f_2 = f_{D+2},
\end{equation}
which holds true for generalized Fibonacci sequences $(f_n)_{n} = \Fib(a,b)$ for any $a,b$, we simplify this further such that
\begin{equation} \label{eq:Omegaw:Fib:bound}
\begin{split}
   |\Omega_{w}^{\Fib(a,b)}| &= 2DN - 2\sum_{j=1}^Df_j\\
    &= 2DN - 2(f_{D+2} - f_2)\\
    &< 2DN - 2w +2b \leq 2 DN, 
\end{split}
\end{equation} 
where we used that $f_{D+2} \geq f_{D+1} > w$ in the first inequality and $b \leq w$ in the last inequality. Next,  we find a bound on the index $D$ of the largest sequence element $f_D$ such that $f_D\leq w$. We solve for $D$ such that  $f_D\le w.$ Using equation \eqref{eq:gen:Binet:2} of Lemma \ref{lemma:generalized:binet}, we observe that
\begin{equation*}
\begin{split}
    w &\geq f_D = \frac{b-a\psi}{\sqrt{5}}\phi^{D-1}+\frac{a\phi-b}{\sqrt{5}}\psi^{D-1} \\
    &\geq \frac{b-a\psi}{\sqrt{5}}\phi^{D-1}-\frac{|a\phi-b|}{\sqrt{5}}|\psi|^{D-1} \\
    &\geq \frac{b-a\psi}{\sqrt{5}}\phi^{D-1}-\frac{|a\phi-b|}{\sqrt{5}}
    \end{split}
\end{equation*}
using that  $|\psi|=\left|\frac{1-\sqrt{5}}{2}\right|\le 1$ in the last inequality.
Solving the latter inequality for $D$, we obtain the bound
$$
D\le \frac{\log(\sqrt{5}w+|a\phi - b|) - \log(b-a \psi)}{\log\phi}+1, 
$$
using that $\phi > 1 $, $\psi < 0$ and $1 \leq a \leq b$. 
Inserting this bound into \eqref{eq:Omegaw:Fib:bound}, this results in the total bound
\begin{equation*}
\begin{split}
|\Omega_{w}^{\Fib(a,b)}| \leq  2N \left(\frac{\log(\sqrt{5}w+|a\phi - b|) - \log(b-a \psi)}{\log\phi}+1\right).
\end{split}
\end{equation*}
\end{proof}

\subsection{Time Complexity Bound for \modelname}\label{sec:bound:Fibottention:complexity}
Lemma \ref{lemma:complexity} can be used to quantify the sparsity of each head attention matrix used in the \modelname~modification of MHSA. Specifically, we recall from Section \ref{sec:fibottention} that \modelname~uses $h$ different sparsity patterns $\Omega_{w_1}^{\Fib(a_1^{\text{Wyt}},b_1^{\text{Wyt}})},\Omega_{w_2}^{\Fib(a_2^{\text{Wyt}},b_2^{\text{Wyt}})},\cdots, \Omega_{w_h}^{\Fib(a_h^{\text{Wyt}},b_h^{\text{Wyt}})}$, where the initial two generalized Fibonacci sequence elements are given by $a_i^{\text{Wyt}} = \lfloor \lfloor i \phi\rfloor \phi \rfloor$ and $b_i^{\text{Wyt}} = \lfloor\lfloor i \phi\rfloor \phi^2 \rfloor$, where $i = 1,\ldots ,h$ is a head index and $\phi = \frac{1+\sqrt{5}}{2}$ is the golden ratio (corresponding to the Wythoff array, see Table \ref{tab:mod-wythoff-array} for an illustration). Furthermore, the window size bounds $w_1,w_2,\cdots,w_h$ are chosen to interpolate between $\wmin$ and $\wmax$ based on the formula 
\[
w_i = \wmin + \Big\lfloor \frac{\wmax-\wmin}{h-1} (i-1) \Big\rfloor,
\]
for all $i=1,\ldots,h$. 

In particular, we obtain the following result.
\begin{theorem} \label{thm:complexity}
Assume that \modelname~is used in a Transformer block with $N$ tokens of dimension $d$ which contains $h$ heads. Then the time complexity of computing all necessary query-key dot products in \modelname~can be bounded by
\[
\begin{split}
\sum_{i=1}^h \frac{d}{h} |\Omega_{w_i}^{\Fib(a_i^{\text{Wyt}},b_i^{\text{Wyt}})}| 
&\leq 2N d \left(2.08 \log((\sqrt{5}+1)\wmax)-1\right) \\
&\leq 4.16 \cdot\! N d \log(3.3\cdot\! N),
\end{split}
\] 
where $\log(\cdot)$ is the natural logarithm.
\end{theorem}
\begin{proof}
    Fix a head index $i$ and let $m_i = \lfloor i \phi \rfloor$. By definition, we have that $ a_i^{\text{Wyt}} = \lfloor m_i \phi \rfloor$ and  $ b_i^{\text{Wyt}} = \lfloor m_i \phi^2 \rfloor$. Since the golden ratio $\phi$ satisfies $\phi^2 = \phi + 1$, we see that 
    \[
    b_i^{\text{Wyt}} = \lfloor m_i \phi^2 \rfloor = \lfloor m_i \phi + m_i \rfloor  = a_i^{\text{Wyt}}  + m_i.
    \]
    Therefore,
\[
a_i^{\text{Wyt}}\phi-b_i^{\text{Wyt}}
= a_i^{\text{Wyt}}(\phi-1)-m_i
= \frac{a_i^{\text{Wyt}}}{\phi}-m_i.
\]
From $a_i^{\text{Wyt}}=\lfloor m_i\phi\rfloor$ we have
$a_i^{\text{Wyt}}\le m_i\phi < a_i^{\text{Wyt}}+1$, hence
\[
\frac{a_i^{\text{Wyt}}}{\phi}\le k_i < \frac{a_i^{\text{Wyt}}+1}{\phi}
= \frac{a_i^{\text{Wyt}}}{\phi}+\frac{1}{\phi},
\]
which implies
\[
-\frac{1}{\phi} < \frac{a_i^{\text{Wyt}}}{\phi}-k_i \le 0
\]
and thus,
\[
\big|a_i^{\text{Wyt}}\phi-b_i^{\text{Wyt}}\big|\le \frac{1}{\phi}<0.62
\]
due to the value of the golden ratio $\phi$. 
Moreover, since $\psi=(1-\sqrt 5)/2$ from Lemma \ref{lemma:generalized:binet} satisfies $\psi =-1/\phi<0$ and $a_i^{\text{Wyt}},b_i^{\text{Wyt}}\in\N$, we find the lower bound
\[
\begin{split}
&b_i^{\text{Wyt}}-a_i^{\text{Wyt}}\psi
= b_i^{\text{Wyt}}+\frac{a_i^{\text{Wyt}}}{\phi}
\ge \left(a_i^{\text{Wyt}}+1\right)+\frac{a_i^{\text{Wyt}}}{\phi} \\
&= a_i^{\text{Wyt}}\left(1+\frac{1}{\phi}\right)+1
= a_i^{\text{Wyt}}\phi+1
\ge \phi+1
= \phi^2,
\end{split}
\]
using again the quadratic golden ratio equation and the fact that $a_i^{\text{Wyt}} \geq 1$ for any $i$.
Applying Lemma~\ref{lemma:complexity} to head $i$ yields
\[
\begin{split}
&\big|\Omega_{w_i}^{\Fib(a_i^{\text{\rm Wyt}},b_i^{\text{\rm Wyt}})}\big|
\\&\le
2N\left(
\frac{\log(\sqrt5\,w_i+0.62 )
-\log(\phi^2)}{\log\phi}
+1
\right) \\
&=2N\left(
\frac{\log(\sqrt5\,w_i+0.62)}{\log \phi}
-1
\right).
\end{split}
\]
Now, we observe that the dimension of the query and key vectors $(Q_i)_{j,:}, (K_i)_{j,:} \in \R^{d_h}$ is $d_h = d/h$ for each head, which is why we can bound the number of operations to compute \modelname's sparse attention matrices $\cA_1^{\Omega_{w_1}^{\Fib(a_1^{\text{Wyt}},b_1^{\text{Wyt}})}}, \ldots, \cA_h^{\Omega_{w_h}^{\Fib(a_h^{\text{Wyt}},b_h^{\text{Wyt}})}}$ (see \eqref{eq:Omegamasking}) as
\[
\frac{d}{h} \sum_{i=1}^h\big|\Omega_{w_i}^{\Fib(a_i^{\text{\rm Wyt}},b_i^{\text{\rm Wyt}})}\big| \leq 2Nd\left(\frac{\log(\sqrt5\,\wmax+0.62)}{\log\phi}-1\right),
\]
using that $w_i \leq \wmax$ for any $i = 1,\ldots, h$ and the monotonicity of $\log(\cdot)$. Since $1/\log\phi \approx 2.078086\ldots \le 2.08$, the first bound of Theorem \ref{thm:complexity} follows.

Using then the fact that $1+\sqrt{5} \leq 3.3$ and $\wmax \leq N$ yields the final bound of  Theorem \ref{thm:complexity}.
\end{proof}
To show this result, we used $b \leq w_i$ and $\wmax \leq N$.
The statement of Theorem \ref{thm:complexity} implies that a forward pass of \modelname~ (default Wythoff version) has a time complexity of $O(N \log(N))$ with respect to the number of tokens. With only cosmetic adjustments, Theorem \ref{thm:complexity} and the $O(N \log(N))$ attention complexity hold true for the modified Wythoff variant of \modelname.

\section{Algorithmic Outline for \modelname} \label{sec:Fibottention:pseudocode}
In this section, we provide for completeness detailed pseudo code that facilitates the modification of a MHSA module by \modelname. Algorithm~\ref{algo1} generates the generalized Fibonacci sequence elements of $\Fib(a,b)$ given initial elements $a$ and $b$ up to a window size bound $w$. Algorithm~\ref{alg:Fibomask} computes the support sets $\Omega_1,\ldots,\Omega_h$ of \modelname as described in Section \ref{sec:fibottention} for each of the $h$ attention heads, given $L$ attention layers in a Transformer architecture (aggregated in tensor $\Omega$), also incorporating dense interactions with the class token (in accordance with standard practice for ViTs). The flag \emph{$\text{is\_modified}$} indicates whether the Modified Wythoff variant of \modelname is utilized or not.

Finally, Algorithm~\ref{alg:Fibottention} demonstrates how \modelname\ can be implemented in Multi-Head Self-Attention (MHSA) to compute the attention mechanism.

\begin{algorithm}[H]
\caption{\modelname in a Vision Transformer block.}
\label{alg:Fibottention}
\begin{algorithmic}[1]
\State \textbf{Input:} $X \in \mathbb{R}^{(N+1)\times d}$
\State \textbf{Output:} $O \in \mathbb{R}^{(N+1)\times d}$
\State \textbf{Parameters:} $\{W_i^Q,W_i^K,W_i^V\}_{i=1}^h,\; W^Z$, with $W_i^{Q,K,V}\in\mathbb{R}^{d\times d_h}$, $d_h=\frac{d}{h}$
\State \textbf{Hyperparameters:} $\wmin,\wmax,\textit{is\_modified}$

\State $\Omega \gets \textbf{getMask}(L,N,h,\wmin,\wmax,\textit{is\_modified})$ \Comment{Alg.~\ref{alg:Fibomask}}

\For{$i=1$ \textbf{to} $h$}
    \State $S_i \gets (N \times N)$ matrix with $-\infty$ entries
    \State $Q_i \gets X W_i^Q,\;\; K_i \gets X W_i^K,\;\; V_i \gets X W_i^V$
    \State $S_i[\Omega[i,\omega_1,\omega_2]] \gets \frac{(Q_i)_{\omega_1}^{\top}(K_i)_{\omega_2}}{\sqrt{d_h}}$ for all $(\omega_1,\omega_2) \in [N]^{2}$ with $\Omega[i,\omega_1,\omega_2] = 1$.
    \State $A_i \gets \operatorname{softmax}(S_i)$
    \State $Z_i \gets A_i V_i \in \mathbb{R}^{(N+1)\times d_h}$
\EndFor

\State $Z \gets \operatorname{Concat}(Z_1,\dots,Z_h) \in \mathbb{R}^{(N+1)\times (h d_h)}$
\State $O \gets Z W^Z$, where $W^Z \in \mathbb{R}^{(h d_h)\times d}$
\State \Return $O$
\end{algorithmic}
\end{algorithm}

\begin{algorithm}[H]
\caption{\textbf{getFibonacci,} Computation of Generalized Fibonacci Sequence $\Fib(a,b)$ Elements up to $w$.}
\label{algo1}
\begin{algorithmic}[1]
\State \textbf{Input:} $a,b,w$ \qquad \textbf{Output:} $fib\_seq$
\State $fib\_seq \gets [a,b]$
\While{$fib\_seq[-1] + fib\_seq[-2] \le w$}
    \State $next\_num \gets fib\_seq[-1] + fib\_seq[-2]$
    \State append $next\_num$ to $fib\_seq$
\EndWhile
\State \Return $fib\_seq$
\end{algorithmic}
\end{algorithm}

\begin{algorithm}[h]
\caption{\textbf{getMask}, Generation of Sparsity Patterns of \modelname.}
\label{alg:Fibomask}
\begin{algorithmic}[1]
\State \textbf{Input:} $L,N,h,\wmin,\wmax,\text{is\_modified}$
\State \textbf{Output:} $\Omega \in \{0,1\}^{h\times (N+1)\times (N+1)}$
\State $\phi \gets \frac{1+\sqrt{5}}{2}$
\State $\Omega \gets 0^{h\times (N+1)\times (N+1)}$
\For{\textbf{each} head $i \in \{1,\cdots,h\}$}
    \State $w_i \gets \wmin$ \If{$h>1$} $+\left\lfloor \frac{(i-1)(\wmax-\wmin)}{h-1}\right\rfloor$ \EndIf
    \State $a \gets \lfloor \lfloor i\phi \rfloor \phi \rfloor,\;\; b \gets \lfloor \lfloor i \phi \rfloor\phi^2 \rfloor$ \Comment{Wythoff pair for head $i$}
    \If{$\text{is\_modified}$}
        \State $b_{\text{Wyt-m}} \gets b-a$;\;\; $a_{\text{Wyt-m}} \gets a-b_{\text{Wyt-m}}$
        \State $I \gets \textbf{getFibonacci}(a_{\text{Wyt-m}},b_{\text{Wyt-m}},w_i)$
    \Else
        \State $I \gets \textbf{getFibonacci}(a,b,w_i)$
    \EndIf

    \State $\Theta \gets 0^{(N+1)\times (N+1)}$
    \State $\Theta_{0,:}\gets 1$;\;\;$\Theta_{:,0}\gets 1$ \Comment{Keep class-token connections}
    \For{\textbf{each} $o \in I$}
        \If{$o \le 0$} \State \textbf{continue} \EndIf
        \For{\textbf{each} $j \in \{0,\cdots,N-o\}$}
            \State $(\Theta)_{j,\,j+o}\gets 1$;\;\; $(\Theta)_{j+o,\,j}\gets 1$ \Comment{Allow $\pm o$ offsets}
        \EndFor
    \EndFor
    \State $\Omega[i,:,:]\gets \Theta$
\EndFor

\State $\Omega \gets \textbf{randomshuffle}(L,\Omega)$ \Comment{Assign masks across $L$ layers}
\State \Return $\Omega$
\end{algorithmic}
\end{algorithm}

\section{Further Implementation Details}\label{sec:implementation_details}
In this section, we provide a comprehensive outline for the implementation of \modelname\,\! within a multi-head self-attention block of a transformer architecture.

\subsection{Integration of \modelname~in Variants of ViTs} 
\label{sec:app:ViTintegration}
As we argued in Section \ref{sec:method}, \modelname can be considered as a drop-in replacement of MHSA that can be used within different vision Transformer architectures. In Section \ref{sec:experimental:results}, we provided empirical results about the integration of \modelname into various ViTs, for which we provide implementation details below: 
For the Swin-B \cite{liu2021swin}  experiment of Table \ref{tab:backbones}, we replace  self-attention of Swin-B with \modelname\ only in the first two stages of the model. The last two stages of the Swin-B, which are less computationally intensive due to prior patch merging modules, remain unmodified. We follow the standard training procedure of Swin-B~\cite{liu2021swin}.
ConViT-B~\cite{d2021convit} consists of gated positional self-attention (GPSA) and MHSA blocks. We apply \modelname\ only to replace the MHSA blocks and train the model  following~\cite{d2021convit}. As iFormer~\cite{zheng2025iformer_iclr} uses single-head self-attention instead of MHSA, we replace its only attention head's attention matrix with $\cA_1^{\Omega_{\wmax}^{\Fib(a_1^{\text{Wyt}},b_1^{\text{Wyt}})}}$. To evaluate \modelname\ within the UPop~\cite{shi2023upop} framework, we first replace the standard MHSA blocks of the target vision-language backbone with \modelname\ and subsequently apply UPop's unified and progressive pruning search to compress the modified architecture.

\subsection{Experimental Configuration for Image Classification}
The training settings for all image classification experiments performed in Section \ref{sec:numerical} are detailed in Table \ref{tab:in1k_settings}. We conducted all these experiments for 100 epochs using a batch size of 64 on 4 RTX A6000 GPUs.

Regarding the sparse attention baselines presented in Tables \ref{tab:sota} and \ref{tab:pruning_results}, BigBird~\cite{zaheer2020big} and Sparse Transformer~\cite{SparseTransformer19} denote adaptations of their respective sparse attention schemes applied to ViT-B. For BigBird, we utilize the specific hyperparameter configuration (local window size, global tokens, and random interactions) justified by our ablation study in Appendix D.B, which selects the variant that yields the optimal trade-off between accuracy and efficiency. For Sparse Transformer, we employ the \textit{strided} attention variant~\cite{SparseTransformer19}. Both baselines are configured to operate under pruning ratios comparable to \modelname, ensuring that observed performance differences primarily arise from how the remaining token interactions are structured rather than disparities in computational cost.

\begin{table}[ht]
    \centering
    \caption{C10, C100, and Tiny-IN Training Settings~\cite{deit}.}
    \begin{tabular}{@{}l@{\hspace{8pt}}p{0.3\linewidth}@{}}
    \toprule
        Input Size & 224$\times$224 \\
        Crop Ratio & 0.9  \\
        Batch Size & 64 \\
        \midrule
        Optimizer & AdamW \\
        Optimizer Epsilon & 1.0e-06 \\
        Momentum & 0.9 \\
        Weight Decay & 0.05 \\ 
        Gradient Clip & 1.0 \\
        \midrule
        Learning Rate Schedule & Cosine\\
        Learning Rate & 1e-3 \\
        Warmup LR & 1.0e-6 \\
        Min LR & 1.0e-5 \\
        Epochs & 100 \\
        Decay Epochs & 1.0 \\
        Warmup Epochs & 5 \\
        Decay Rate & 0.988 \\
        \midrule
        Exponential Moving Average (EMA) & True \\
        EMA Decay & 0.99992 \\
        \midrule
        Random Resize \& Crop Scale \& Ratio & (0.08, 1.0),\\
        &(0.67, 1.5)\\
        Random Horizontal Flip Probability & 0.5 \\
        Color Jittering & 0.4 \\
        Auto-augmentation & rand-m15-n2-mstd1.0-inc1\\
        Mixup & True \\
        Cutmix & True \\
        Mixup, Cutmix Probability & 0.5, 0.5 \\
        Mixup Mode & Batch \\
        Label Smoothing & 0.1 \\
    \bottomrule
    \end{tabular}
    \label{tab:in1k_settings}
\end{table}

\section{Ablations for Sparse Attention Mechanisms}\label{sec:further analysis}
In this section, we present ablations exploring the role of different hyperparameter choices within the sparse attention adaptations to ViTs on accuracy, sparsity and computational efficiency, which justify the experimental setups presented in Section \ref{sec:numerical}.


\subsection{Impact of Randomized Sparsity}
\label{subsec:randomized_sparsity}
In Table \ref{tab:random-attention}, we present the top-1 accuracy results of randomly masked (index pairs sampled uniformly at random) self-attention on C10 and C100 datasets under varying levels of sparsity. We observe that as the masking ratio increases, there is a consistent decline in accuracy, highlighting the trade-off between reducing computational cost and maintaining performance. This shows that the performance improvements obtained by the sparse attention mechanism of \modelname reported in Table \ref{tab:sota} are not at all observed for a random sparse sampling pattern, at any pruning ratio. This can be interpreted such that the pattern-less nature of random masking fails to identify critical token relationships underscoring the necessity of structured approaches, leading to significant information loss as sparsity increases.
\begin{table}[ht]
  \caption{Top-1 accuracy on CIFAR-10 (C10) and CIFAR-100 (C100) for Random Attention with and without the class token, under different attention pruning ratios.}
  \label{tab:random-attention}
  \setlength{\tabcolsep}{7pt}
  \small
  \centering
  \resizebox{\linewidth}{!}{%
  \begin{tabular}{ccccc}
    \toprule
    \multirow{2}{*}{\makecell[c]{\textbf{Attention} \\ \textbf{Pruning Ratio (\%) $\uparrow$}}} &
    \multicolumn{2}{c}{\textbf{w/ Class Token}} &
    \multicolumn{2}{c}{\textbf{w/o Class Token}} \\
    \cmidrule(lr{0.5em}){2-3}\cmidrule(lr{0.5em}){4-5}
    & C10 & C100 & C10 & C100 \\
    \midrule
    0\%   & \textbf{83.5} & \textbf{59.3} & \textbf{83.5} & \textbf{59.3} \\
    20\%  & \textbf{83.3} & \textbf{59.1} & 83.2 & 59.0 \\
    40\%  & \textbf{83.2} & 58.7 & 82.8 & \textbf{58.8} \\
    60\%  & \textbf{82.7} & \textbf{58.9} & 81.9 & 58.7 \\
    80\%  & \textbf{82.4} & \textbf{58.5} & 81.1 & 58.1 \\
    90\%  & \textbf{81.6} & \textbf{58.1} & 80.4 & 56.9 \\
    100\% & \textbf{81.4} & \textbf{58.0} & 77.5 & 47.9 \\
    \bottomrule
  \end{tabular}%
  }
\end{table}


\subsection{Impact of Structured Sparsity}
In this section, we present the role  
different hyperparameter choices within a BigBird \cite{zaheer2020big} for image classification experiments. In particular, we present in Table \ref{tab:bigbird_results} how choices of the local window size ($w$), global token interactions ($g$), and the number of random token interactions ($r$), impact ViT-B performance when trained on C10. 
We observe that $w = 2$, $g = 1$, and $r = N$ result in a model accuracy of 85.41\%. However, further increasing randomness to $r = 2N$ reduced accuracy to 84.75\%. This decline can be attributed to the dilution of critical token relationships in visual data, where spatially correlated information plays a vital role.
Expanding the local window size to $w = 4$ resulted in a peak accuracy of 86.33\%, demonstrating that larger windows capture richer token dependencies and compensate for reduced reliance on randomness. We note that a further increase of $w$ would lower the masking ratio, increasing computational costs and diminishing efficiency. This trade-off highlights the limitations of BigBird compared to \modelname.

\begin{table}[h]
\caption{Performance comparison of masking strategies in BigBird with varying configurations. Top-1 accuracy is reported.}
\label{tab:bigbird_results}
\centering
\setlength{\tabcolsep}{10.7pt}
\resizebox{\linewidth}{!}{%
\begin{tabular}{lcc}
\toprule
\textbf{Configuration} & \textbf{Mask Ratio} & \textbf{C10} \\
\midrule
$w_i = 2 \; | \; g = 1 \; | \; r = N$   & 96.97 & 85.41 \\
$w_i = 2 \; | \; g = 1 \; | \; r = 2N$  & 96.47 & 84.75 \\
$w_i = 4 \; | \; g = 1 \; | \; r = N$   & 94.21 & \textbf{86.33} \\
\bottomrule
\end{tabular}%
}
\end{table}

\section{Further Ablation Studies of \modelname~}\label{sec:further_ablations}
In this section, we evaluate additional architectural choices in \modelname through controlled experiments, supplementing the ablations presented in Section \ref{sec:further_ablations}.

\subsection{Semantic Impact of Principal Diagonal in Self-Attention}
In Table~\ref{tab:local-w-wo-diagonal}, we conduct an ablation study on the use of fixed local window sparsity pattern $\Omega_{w} = \big\{(j,k) \in \{1,\cdots,N\}^2: |j-k| \leq w \big\}$ with window size $w$ (identical across all heads) for attention computation, comparing configurations with (usage of $\Omega_w)$ and without the principal diagonal (usage of $\Omega_w \setminus \{(j,j) \in \{1,\cdots, N\}^2, j \in \{1,\cdots, N\} \}$). Table~\ref{tab:local-w-wo-diagonal} suggests that the classification accuracy increases when removing the principal diagonal from $\Omega_w$, consistently for different choices of $w$, confirming an observation of SparseBERT~\cite{shi2021sparsebert} in our setting. 
The pruning ratio is defined as the proportion of token interactions, $\left(\frac{N^2 - |\Omega_w|}{N^2}\right) \cdot 100 \%$
that are excluded during the attention computation. A larger pruning ratio corresponds to fewer token interactions being utilized, which has the potential to reduce computational costs.
\begin{table}[ht]
  \caption{Top-1 accuracy of ViT-B on C10 and C100 for windowed self-attention with and without the principal diagonal. For each window size $w_i$ we report the attention pruning ratio and accuracy when the diagonal entries are kept (left) or removed (right).}
  \label{tab:local-w-wo-diagonal}
  \setlength{\tabcolsep}{4.5pt}
  \small
  \centering
  \resizebox{\linewidth}{!}{%
  \begin{tabular}{ccccccc}
    \toprule
    \multirow{2}{*}{\textbf{$w$}} &
    \multicolumn{3}{c}{\textbf{w/ Main Diagonal}} &
    \multicolumn{3}{c}{\textbf{w/o Main Diagonal}} \\
    \cmidrule(lr{0.5em}){2-4}\cmidrule(lr{0.5em}){5-7}
    & Pruning Ratio & C10 & C100 & Pruning Ratio & C10 & C100 \\
    \midrule
    2   & 97.46 & \textbf{86.1} & 62.0 & \textbf{97.97} & 85.7 & \textbf{62.2} \\
    10  & 89.57 & 86.8 & 62.4 & \textbf{90.08} & \textbf{87.0} & \textbf{63.4} \\
    15  & 84.81 & 87.5 & 64.7 & \textbf{85.32} & \textbf{88.0} & \textbf{64.8} \\
    20  & 80.17 & 87.9 & \textbf{64.9} & \textbf{80.69} & \textbf{88.0} & \textbf{64.9} \\
    40  & 62.94 & 87.6 & 64.5 & \textbf{63.45} & \textbf{87.7} & \textbf{65.0} \\
    \bottomrule
  \end{tabular}%
  }
\end{table}

\begin{table}
\centering
\small
\setlength{\tabcolsep}{10pt}
\renewcommand{\arraystretch}{1.05}

\resizebox{\linewidth}{!}{%
\begin{tabular}{>{\centering\arraybackslash}p{1.2cm}ccc}
\toprule
\multirow{2}{*}{\textbf{$w_i$}} &
\multicolumn{2}{c}{\textbf{Top-1 Accuracy (\%)}} &
\multirow{2}{*}{\makecell[c]{\textbf{Attention} \\ \textbf{Pruning Ratio} $\uparrow$}} \\
\cmidrule(lr){2-3}
& C10~\cite{cifar} & C100~\cite{cifar} & \\
\midrule
2   & 85.7 & 62.2 & 97.97\% \\
3   & 86.7 & 62.9 & 96.97\% \\
4   & 86.8 & 62.9 & 95.97\% \\
5   & 86.7 & 63.0 & 94.98\% \\
6   & 86.5 & 62.9 & 93.99\% \\
7   & 86.9 & 62.5 & 93.00\% \\
8   & 86.3 & 63.1 & 92.02\% \\
9   & 86.8 & 62.9 & 91.05\% \\
10  & 86.9 & 63.4 & 90.08\% \\
15  & \textbf{88.0} & 64.8 & 85.32\% \\
20  & \textbf{88.0} & 64.9 & 80.69\% \\
40  & 87.7 & \textbf{65.0} & 63.45\% \\
80  & 87.0 & 62.9 & 35.24\% \\
120 & 85.7 & 61.5 & 15.35\% \\
160 & 83.7 & 60.2 & 3.79\% \\
196 & 83.5 & 59.3 & 0\% \\
\bottomrule
\end{tabular}%
}

\caption{Top-1 accuracy of ViT-B with windowed self-attention on CIFAR-10~\cite{cifar} and CIFAR-100~\cite{cifar} as a function of the shared window size $w_i = \wmin = \wmax$.}
\label{tab:window_sizes}
\end{table}

\subsection{Choice of Shared Window Size Limit Among Heads} \label{sec:app:sharedwindowsize}
Table \ref{tab:window_sizes} reports the top-1 accuracy of ViT-B on C10 and C100 as a function of the shared window size $w_i=\wmin=\wmax$. We observe that increasing the window size from very small neighborhoods up to a moderate range steadily improves performance, with the best results obtained for $w_i \in [15,20]$. This indicates that incorporating local context within a limited spatial extent is beneficial for learning discriminative token representations. However, as the window size continues to grow, accuracy gradually declines alongside the attention pruning ratio. This suggests that excessively large windows dilute locality information and reduce the specialization of attention patterns. Overall, Table \ref{tab:window_sizes} highlights that moderate window sizes provide the best trade-off between local focus and contextual coverage.


\subsection{Analysis of Corrupted Datasets}\label{sec:corruption}
In Section \ref{subsec:robustness}, we reported results on the performance of \modelname~in the context of corrupted datasets. The top-$1$ accuracy numbers reported in Figure \ref{fig:corruption} contain corrupted data covering 19 corruption types, which include brightness, contrast, defocus blur, elastic transform, fog, frost, Gaussian blur, Gaussian noise, glass blur, impulse noise, JPEG compression, motion blur, pixelation, saturation, shot noise, snow, spatter, speckle noise, and zoom blur. In Figure~\ref{fig:corruption_c10_c100}, we present a more detailed breakdown of the performance of \modelname-equipped ViT-Base compared to ViT-Base using MHSA on C10 and C100 corrupted datasets~\cite{hendrycks2019robustness} across the different corruption types. The accuracies in Figure~\ref{fig:corruption_c10_c100} are averaged across all five severity levels for each corruption type, providing a comprehensive view of robustness under a wide range of challenging conditions. We observe that \modelname consistently outperforms ViT-Base on both datasets for every corruption type. Finally, we present in Figure~\ref{fig:corruption_c10} the C10 results for four representative corruptions across all five severity levels, illustrating that the improvements of \modelname persist as corruption severity increases.

\begin{figure}[t]
\begin{floatrow}[2]
\ffigbox[2.05\textwidth]{%
\centering
\includegraphics[width=\linewidth]{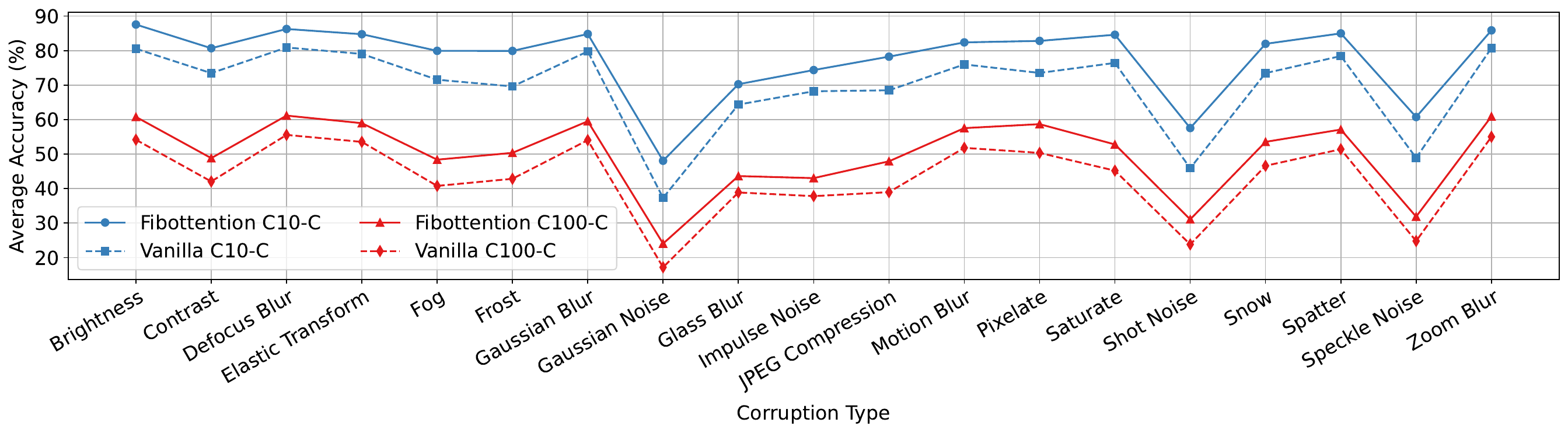}
}{%
\caption{Performance of \modelname~compared to Vanilla ViT on C10 and C100 corrupted datasets.}
\label{fig:corruption_c10_c100}
}
\end{floatrow}
\end{figure}

\begin{figure}[h]
    \centering
    \includegraphics[width=\textwidth]{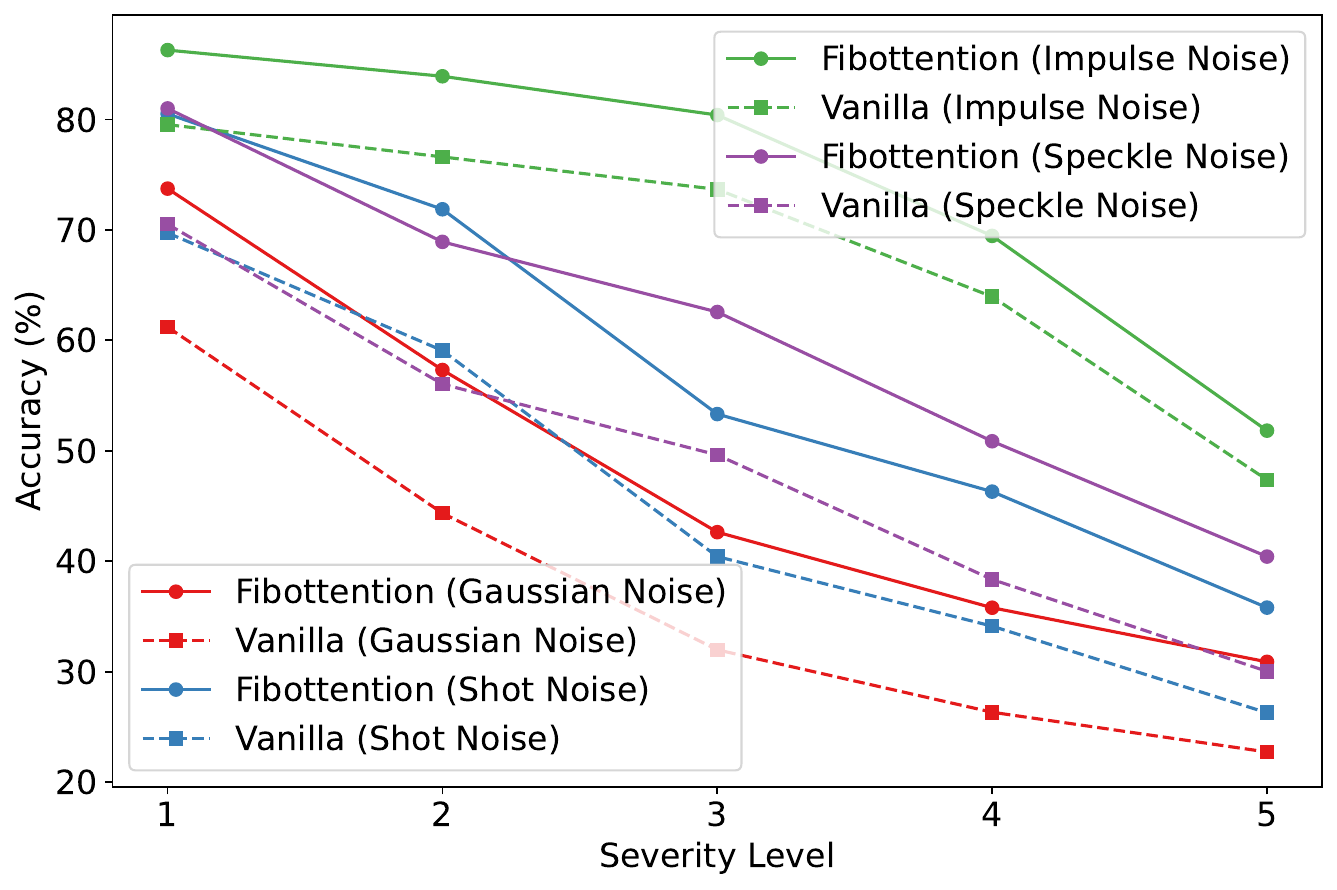}
    \caption{Performance of \modelname ViT-B and Vanilla ViT-B on C10 corrupted dataset under four types of corruption at five levels of severity.}
    \label{fig:corruption_c10}
\end{figure}


\end{document}